\documentclass[review]{elsarticle}

\usepackage{lineno,hyperref}
\modulolinenumbers[5]

\journal{Elsevier}

\usepackage{epstopdf}
\usepackage{url}
\usepackage[bottom]{footmisc}

\usepackage{makeidx}         
\usepackage{graphicx}        
\usepackage{multicol}        

\usepackage{amsmath,bm,amsfonts,amssymb}
\usepackage{booktabs}

\usepackage{commath}      

\usepackage{subfloat}
\usepackage{color,xcolor,ucs}
\usepackage[font=small,labelfont=bf]{caption}
\usepackage{tabularx}
\usepackage{float}

\usepackage{amsthm}

\usepackage{subfig}

\usepackage{wrapfig}
\usepackage[ruled,vlined,linesnumbered,algo2e]{algorithm2e}

\usepackage{multirow}
\usepackage{listings}
\usepackage{lipsum}
\usepackage{textcomp} 
\usepackage{algorithm}
\usepackage{algpseudocode}

\newtheorem{lemma}{Lemma}%

\newtheorem{definition}{Definition}%
\SetKwInOut{Input}{require}
\SetKwInOut{Output}{output}
\SetKwProg{When}{when}{}{}

\makeatletter
\renewcommand*{\ALG@name}{Workflow}
\makeatother

\definecolor{revisioncolor}{rgb}{0.1,0.1,1}

\makeatletter
\def\ps@pprintTitle{%
 \let\@oddhead\@empty
 \let\@evenhead\@empty
 \let\@oddfoot\@empty
 \let\@evenfoot\@empty}
\makeatother

\begin{document}

\begin{frontmatter}

\title{A Task and Motion Planning Framework Using Iteratively Deepened AND/OR Graph Networks}

\author[label1]{Hossein Karami}
\ead{hossein.karami@edu.unige.it}
\address[label1]{Konecranes
}

\author[label2]{Antony Thomas}
\ead{antony.thomas@iiit.ac.in}
\address[label2]{Robotics Research Center, IIIT Hyderabad, Hyderabad 500032, India.}

\author[label3]{Fulvio Mastrogiovanni}
\ead{fulvio.mastrogiovanni@unige.it}
\address[label3]{Department of Informatics, Bioengineering, Robotics, and Systems Engineering, University of Genoa, Via All'Opera Pia 13, 16145 Genoa, Italy. }

\begin{abstract}
In this paper, we present an approach for integrated task and motion planning based on an AND/OR graph network, which is used to represent task-level states and actions, and we leverage it to implement different classes of task and motion planning problems (TAMP).
Several problems that fall under task and motion planning do not have a predetermined number of sub-tasks to achieve a goal.
For example, while retrieving a target object from a cluttered workspace, in principle the number of object re-arrangements required to finally grasp it cannot be known ahead of time. 
To address this challenge, and in contrast to traditional planners, also those based on AND/OR graphs, we grow the AND/OR graph at run-time by progressively adding sub-graphs until grasping the target object becomes feasible, which yields a network of AND/OR graphs. 
The approach is extended to enable multi-robot task and motion planning, and
(i) it allows us to perform task allocation while coordinating the activity of a \textit{given} number of robots, and 
(ii) can handle multi-robot tasks involving an \textit{a priori} unknown number of sub-tasks. 

The approach is evaluated and validated both in simulation and with a real dual-arm robot manipulator, that is, Baxter from Rethink Robotics.
In particular, for the single-robot task and motion planning, we validated our approach in three different TAMP domains.
Furthermore, we also use three different robots for simulation, namely, Baxter, Franka Emika Panda manipulators, and a PR2 robot. 
Experiments show that our approach can be readily scaled to scenarios with many objects and robots, and is capable of handling different classes of TAMP problems.
\end{abstract}


\end{frontmatter}

\section{Introduction}
\label{sec:intro}

In common situations, humans trivially perform complex manipulation tasks, such as picking up a tool from a cluttered toolbox, or grabbing a book from a shelf, if necessary by re-arranging occluding objects. 
For humans, these tasks seem to be routine, and they neither require much \textit{conscious} planning nor cognitive focus during action execution. 
Yet, for robots, this is definitely not the case. 
Such complex manipulation tasks as \textit{picking from clutter} or \textit{rearranging a set of objects} require advanced forms of reasoning to decide which objects to pick up or re-arrange, and in which sequence, so as to synthesize motions towards the target objects to account for the geometry-level feasibility of the task-level actions. 
This interaction between task-level symbolic reasoning and the geometry-level motion planning is the subject of integrated \textit{Task and Motion Planning} (TAMP)~\cite{lagriffoul2018RAL}.
In a general sense, TAMP refers to the problem of \textit{jointly} planning a sequence of discrete actions and the corresponding set of feasible continuous motions to attain a formally specified goal state, which could be expressed as a hybrid geometric-symbolic description of facts concerning the robot workspace, including objects therein~\cite{garrett2021ARC,guo2023ACM}.

In the literature, the \textit{de facto} standard syntax for specifying task-level actions is the Planning Domain Definition Language (PDDL)~\cite{mcdermott1998AIPS}, and most approaches resort to it. 
Task-level domains may also be specified using temporal logics~\cite{he2015ICRA}, or formal languages~\cite{dantam2013TRO}. 
Given a task-level domain, task planning finds a sequence of discrete actions from the current state to a desired goal state, which is usually expressed in symbolic form as a conjunction of predicates~\cite{ghallab2016book}. 
Likewise, motion planning approaches find sequences of collision-free configurations to reach a desired goal configuration, which is represented in geometrical terms or in the configuration space~\cite{latombe1991robot}. 
Therefore, TAMP approaches aim at establishing an appropriate mapping between task-level and motion-level domains. 
Examples of such a mapping include, given a task-level plan, the corresponding motion-level actions, or given a task-level state, the corresponding sequence of feasible geometric configurations.    

\begin{figure*}[t!]
\centering
\subfloat[]{\includegraphics[trim=1cm 3cm 5cm 18cm, clip=true,width=0.35\textwidth]{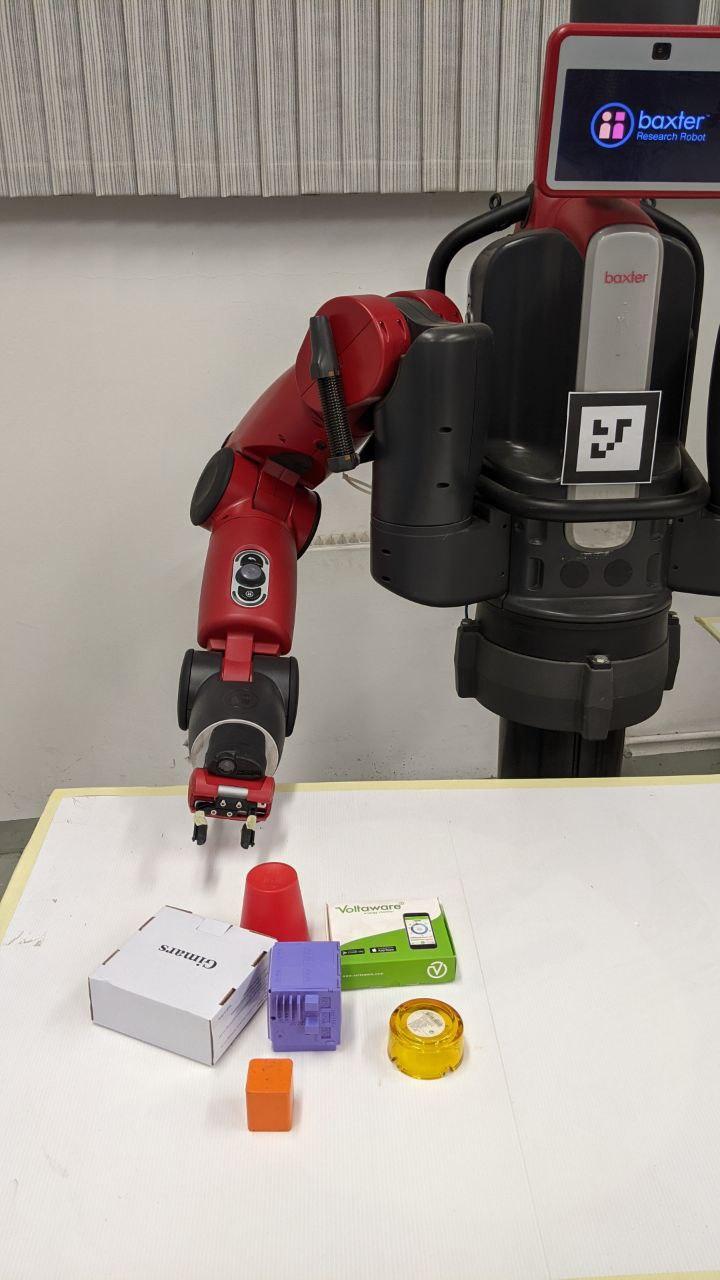}\label{fig:toy1}} \hfill
\subfloat[]{\includegraphics[trim=20cm 5cm 30cm 7cm, clip=true,width=0.595\textwidth]{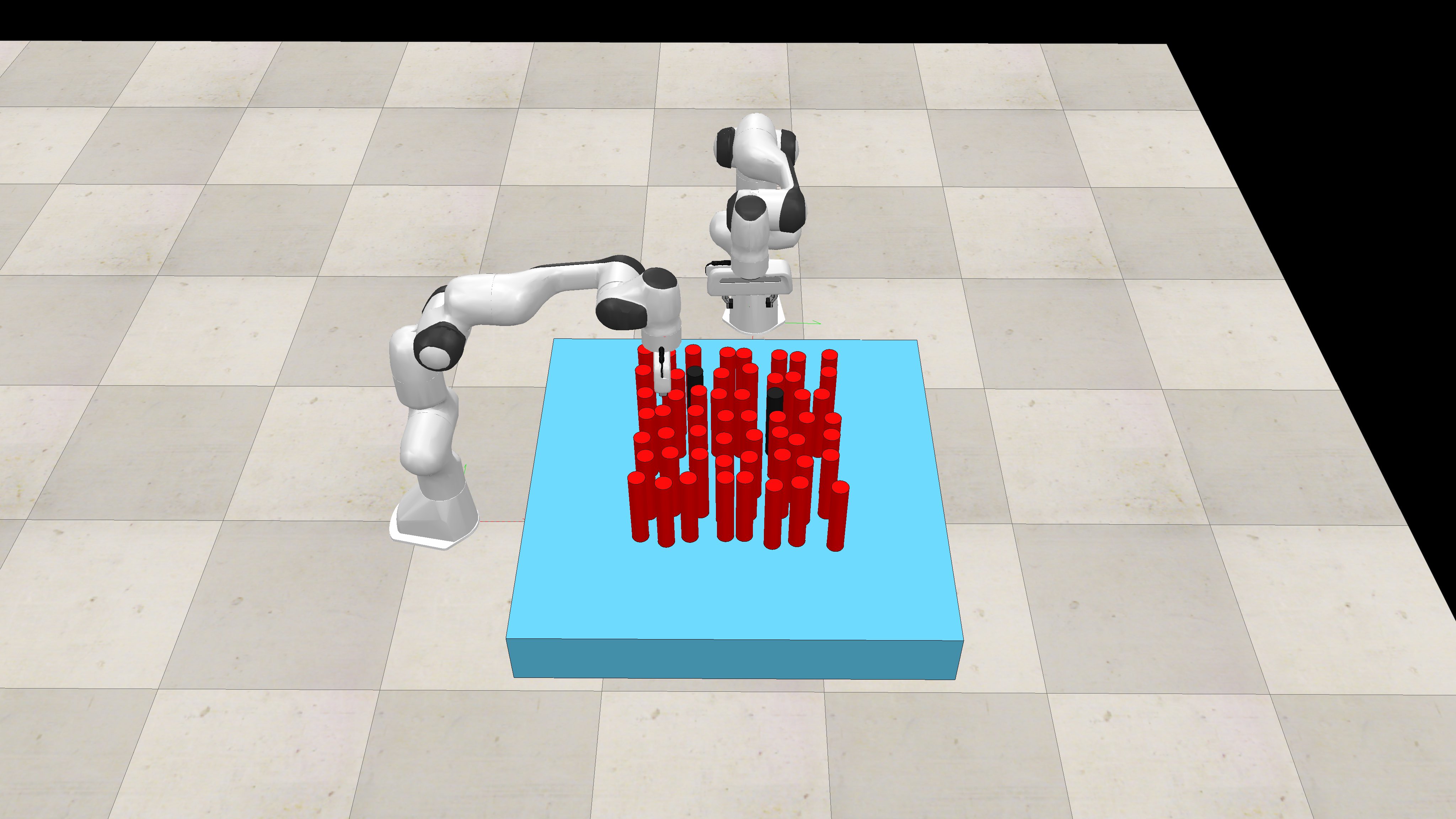}\label{fig:toy2}}
\caption{
(a) The robot needs to pick the purple cube. 
If only side grasps were allowed, then the robot would go through a sequence of pick-and-place actions to de-clutter the area. 
(b) Cluttered table-top with two target objects (in black) and other objects (in red). 
The multi-robot system consists of two manipulators.}
\label{fig:toy}
\end{figure*} 
Though many off-the-shelf PDDL-based planners are available, establishing the correspondence between task-level and motion-level planners is not an easy task, and it requires the development of a full-fledged robot planning and control architecture. 
For the sake of the argument, let us consider a simple \textit{cluttered table-top scenario}, as shown in Fig.~\ref{fig:toy1}, where a robot must pick up a purple cube from clutter. 
As seen in the Figure, under the assumption that only side grasps are allowed, the purple cube cannot be immediately reached, and other objects hindering the \textit{target grasp} need to be properly re-arranged. 
Such a scenario may be modeled using two PDDL actions, namely \textsf{pick} and \textsf{place}, which should be alternated as a sequence of \textit{pick-and-place} schemas to make grasping the purple cube possible.
Obviously enough, in a TAMP formulation, each task-level, symbolic action should also correspond to appropriate, feasible motions. 
Therefore, the typical PDDL domain may contain the following predicate definitions: 
\begin{equation*}
\textsf{(:predicates (clear ?x) (gripper\_empty) (in\_hand ?x))}.
\end{equation*}
The predicates \textsf{(clear\ ?x}), \textsf{(gripper\_empty)}, and \textsf{(in\_hand\ ?x)} respectively would check whether an object, generically identified with the variable \textsf{?x}, would be clear, that is, nothing hinders its grasping, whether the gripper would not hold any object, and which object the gripper would hold. 
In order to keep the scenario simple enough, we deliberately make the assumption that if the target grasp were hindered, the robot would be able to pick up another object, or otherwise different objects in sequence, to make the target \textsf{pick} action feasible. 
We note that PDDL-based planners initiate the search through a process called \textit{grounding}, which is aimed at replacing the variables in the predicates (for example the variable \textsf{?x} above) with possible \textit{anchors} to real-world objects, so as to instantiate predicates and action schemas. 
If our scenario would comprise $6$ possible objects to grasp and/or re-arrange, like in Fig.~\ref{fig:toy1}, that would give rise to $13$ possible propositions, that is, $6$ possibilities for each of the predicates \textsf{(clear\ ?x)} and \textsf{(in\_hand\ ?x)}, as well as one for \textsf{(gripper\_empty)}. 
Therefore, this would yield $n = 2^{13}$ possible states since we consider Boolean truth values.
Finding the shortest path in the search space using a state-transition graph via a typical \textsf{pick} and \textsf{place} scheme would be characterized by a $O(n \log n)$ temporal complexity~\cite{bryce2007AIM}. 
Heuristic search avoids visiting large regions of the transition graph via informed search to reach the symbolic goal state faster. 
It must be noted that an \textit{effective} reaching of the goal state also depends on the configuration space, and that it should be verified that each \textsf{pick} and \textsf{place} action would correspond to feasible motions. 
However, what is even more relevant for this discussion is that, in a general sense, the number of objects that must be re-arranged to pick up the target object may not be known (or easily computed) beforehand at the task-level planning time, which would require observation-based re-planning within the PDDL framework, and in cascade an additional mapping between the PDDL-based action space and the robot observation space~\cite{bertolucci2021TPLP}. 
While PDDL planning is EXPSPACE-complete and can be restricted to less compact propositional encodings to achieve PSPACE-completeness~\cite{helmert2006PHD}, it is noteworthy that the time taken for each re-planning would further exacerbate the temporal complexity.
Focusing instead on TAMP, we observe that although \textit{single-robot} TAMP has been an area of active research, current approaches do not naturally account for the benefits afforded by the presence of \textit{multiple} robots~\cite{motes2020RAL}.
These may include the fact that tasks may be decomposed in a variety of ways, task allocation may involve different robots, or that re-arrangements may benefit from an explicit coordination among robots. 
A straightforward extension of single-robot TAMP approaches to the multi-robot case would have to treat the multi-robot \textit{team} as a collection of (possibly many) single robots, which would further worsen computational and temporal complexity as the number of robots would increase. 

In this paper, we present a probabilistically complete approach for TAMP in single- and multi-robot scenarios, which we refer to as \textit{Task-Motion Planning using Iterative Deepened AND/OR Graph Networks}, and we refer to it as \textsc{TMP-IDAN} for short.
We tackle the two challenges described above, that is 
(i) the computational complexity of typical PDDL-based planners is $O(n \log n)$, and
(ii) the number of sub-tasks may not be known in advance, by defining the task-level domain using an AND/OR graph, and by encoding the task-level abstractions of the TAMP problem in an efficient and compact way within the AND/OR graph. 
With respect to challenge (i) above, in Section~\ref{sec:AOgraph} we show that by using AND/OR graph networks the complexity of task-level planning is \textit{almost linear} with respect to the number of graphs in the network.
In view of challenge (ii), we introduce a variant of AND/OR graph networks that is able to iteratively deepen (expand) at run-time till reaching a state whereby the goal state can be achieved. We then extend our approach to multi-robot TAMP. A preliminary version of multi-robot TAMP was published in the conference paper~\cite{karami2021AIIA}. In the current work, we provide a rigorous formalization of the AND/OR graph networks, offering a more comprehensive explanation of the multi-robot TAMP methodology. Additionally, we present a much more detailed analysis of the experimental results. Specifically, our multi-robot TAMP framework encompasses the following capabilities:  
(i) task allocation among the available robots, and  
(ii) planning a sequence of actions for multiple decomposable tasks using AND/OR graphs, ensuring optimality with respect to a multi-robot wide utility function. 
Finally, the approach is evaluated and validated through various experimental domains, which address challenging aspects of TAMP. 
These include:
\begin{itemize}
    \item
    \textit{Cluttered workspace}.
    This experiment highlights different aspects of TAMP, such as:
    \textit{unfeasible actions}, that is, discrete actions are impossible due to the unfeasibility at the motion planning level, for example because of obstructing objects;
    \textit{large task space}, that is, a substantial search may be required to find a solution;
    \textit{task/motion trade-off}: the problem can be solved in fewer steps if the objects to be rearranged are carefully chosen.
    \item
    \textit{Tower of Hanoi}.
    This problem illustrates challenges such as \textit{unfeasible actions} and the \textit{large task space} issue.
    \item
    \textit{Kitchen}.
    This problem presents challenges such as:
    \textit{unfeasible actions}, as above;
    \textit{non-monotonicity}, that is, some objects may need to be moved more than once;
    \textit{non-geometric actions}, that is, involving actions that change the symbolic states without altering the configuration space, for example, objects may be ``cleaned'' by placing them in the dishwasher or ``cooked'' by placing them in the microwave: the ``clean'' and ``cook'' actions are non-geometric in nature.
\end{itemize}

The rest of the paper is organized as follows. 
In Section~\ref{sec:related} we discuss related state-of-the-art approaches. 
TAMP formalisms are discussed in Section~\ref{sec:background}.
We introduce AND/OR graph networks in Section~\ref{sec:AOgraph}, and conclude the Section by completing the discussion of the toy example in Fig.~\ref{fig:toy}. 
In Section~\ref{sec:approach} we formalize \textsc{TMP-IDAN} for a single as well as a multi-robot scenario. 
We evaluate the capabilities of \textsc{TMP-IDAN} in Section~\ref{sec:results}. 
In Section~\ref{sec:discussion}, we discuss the limitations of \textsc{TMP-IDAN} and Section~\ref{sec:conclusion} concludes the paper.

\section{Related Work}
\label{sec:related}

In the recent past, TAMP has received considerable interest among the Robotics research community~\cite{kaelbling2013IJRR, srivastava2014ICRA, dantam2016RSS, lagriffoul2016IJRR, garrett2018IJRR, thomas2021RAS, garrett2019arxiv}. 
The primary challenge of planning in a hybrid, symbolic-geometrical space is to obtain an efficient mapping between the discrete task level and the continuous motion level. 
A combined search in a hybrid symbolic and geometric space using states composed of both the symbolic and geometric variables is performed in~\cite{cambon2009IJRR}. 
A hierarchical approach for TAMP is introduced in~\cite{kaelbling2012aTR}, wherein planning is done at different levels of abstraction, thereby reducing longer plans to a set of feasible, shorter sub-plans. 
However, planning occurs backwards from the goal (that is, regression), and assumes that motion-level actions are reversible while backtracking. 
A similar approach is also employed in~\cite{pandey2012BIOROB, de2013RSSws} to compute discrete actions with unbounded continuous variables.

Semantic attachments are used in~\cite{dornhege2009SSRR, dornhege2009ICAPS, dantam2016RSS, thomas2021RAS} to associate algorithms to functions and predicate symbols via procedures external to the planning logic. 
Though semantic attachments allow for a mapping between the symbolic and geometric spaces, all the relevant knowledge about the environment is assumed to be known in advance. 
Specifically, the robot configuration and the possible grasp poses need to be specified in advance, which as a matter of fact renders the continuous motion space to be discrete.
The FFRob approach, described in~\cite{garrett2018IJRR}, performs task planning via a search over a finite set of pre-sampled poses, grasps and robot configurations. 
This \textit{a priori} discretization is relaxed in~\cite{kaelbling2011ICRA, srivastava2014ICRA, toussaint2015IJCAI}. 
For example, the approach in~\cite{srivastava2014ICRA} implicitly incorporates motion-level variables while performing symbolic-geometric mapping using a planner-independent interface layer. 
Most approaches (for example~\cite{erdem2011ICRA, dantam2016RSS}) first compute a task-level plan and refine it until a feasible motion plan is found or until timeout. 
To this end, the two approaches in~\cite{erdem2011ICRA, dantam2016RSS} adopt constraint-based task planning to leverage ongoing advances in solvers for Satisfiability Modulo Theories (SMT)~\cite{de2011CACM}. 
The approach in~\cite{thomas2021RAS}, however, checks for the motion feasibility as each task-level action is expanded by the task planner, but in doing so it assumes a discretized motion space and thereby a finite action set. 
Recent work of Garrett \textit{et al.}~\cite{garrett2020ICAPS} addresses this limitation by introducing \textit{streams} within PDDL, which enable procedures for sampling values of continuous variables and thereby encoding a theoretically infinite set of actions.
Functional Object-Oriented Networks (FOON) were introduced in~\cite{paulius2023RAL} to represent knowledge for robots in a way that encourages abstraction for combined task and motion planning.
However, the open-loop nature of the execution cannot handle such aspects as possible collisions with objects.

The above mentioned methods do not consider the implications of having multiple robots in task allocation and collision avoidance, and would have to treat the multi-robot team as a collection of separate robots, which would become intractable as the number of robots increased. 
Current TAMP approaches for multi-arm robot systems focus on coordinated planning strategies, and consider simple pick-and-place or assembly tasks. 
As such, these methods do not scale to scenarios in which complex manipulations are required~\cite{rodriguez2016IROS, umay2019ROSE}.
TAMP for multi-arm robot systems in the context of welding is considered in~\cite{basile2012RCIM}, whereas~\cite{umay2019ROSE} discusses an approach for multi-arm TAMP manipulation. 
However, the considered manipulation tasks reduce to a simple pick-and-place operation involving bringing an object from an initial to a goal position, whereas two tables and a cylindrical object form the obstacles. 
TAMP for multi-robot teams has not been addressed thoroughly, and therefore the literature is not sufficiently developed.
Henkel~\textit{et al.}~\cite{henkel2019IROS} consider multi-robot transportation tasks using a Task Conflict-Based Search (TCBS) algorithm. 
Such an approach solves a combined task allocation and path planning problem, but assigns a single sub-task at a time and hence may not scale well to an increased number of robots. 
Interaction Templates (IT) for robot interactions during transportation tasks are presented in~\cite{motes2019RAL}. 
The interactions enable handing over payloads from one robot to another, but the method does not take into account the availability of robots and assumes that there is always a robot available for such an handover. 
Thus, while considering many tasks at a time this framework does not fare well since a robot may not be immediately available for an handover. 
A distributed multi-robot TAMP method for mobile robot navigation is presented in~\cite{thomas2020STAIRS}. 
However, the authors define task-level actions for a pair of robots and therefore optimal solutions are available for an even number of tasks, and only sub-optimal solutions are returned for an odd number of tasks.
Motes~\textit{et al.}~\cite{motes2020RAL} present TMP-CBS, a multi-robot TAMP approach with sub-task dependencies. 
They employ a CBS method~\cite{sharon2015AI} in the context of transportation tasks. 
Constructing a conflict tree for CBS requires the knowledge of different constraints which depend on the sub-task conflicts, for example, two robots being present at a given location at the same time. 
However, in the table-top scenario considered in this paper, the number of sub-tasks is not known beforehand. 
The capabilities of the discussed multi-robot TAMP methods are summarized in Table~\ref{tab:comp}. 

\begin{table*}[t]
\centering
\scalebox{0.7}{
\begin{tabular}{cccccc} 
\hline
\textit{Parameter}              & TCBS~\cite{henkel2019IROS}    & IT~\cite{motes2019RAL}    & \cite{thomas2020STAIRS}   & TMP-CBS~\cite{motes2020RAL}   & TMP-IDAN\\
\hline
\hline
Task allocation                 & \checkmark                    & \checkmark                & \checkmark                & \checkmark                    & \checkmark\\ 
Task decomposition              &                               & \checkmark                &                           & \checkmark                    & \checkmark\\
Motion planning                 &                               & \checkmark                & \checkmark                & \checkmark                    & \checkmark\\
Unknown number of sub-tasks     &                               &                           &                           &                               & \checkmark\\
\hline
\end{tabular}}
\caption{Comparison of different multi-robot TAMP methods.}
\label{tab:comp}
\end{table*}

\section{Formalization of the Task-Motion Planning Problem}
\label{sec:background}

Task planning (or classical planning) is the process of finding a discrete sequence of actions from the current to a desired goal state~\cite{ghallab2016book}, which is expressed in symbolic form as a proper conjunction of grounded predicates.
For the sake of terminology, in this Section we skim through a number of definitions that are relevant for the discussion that follows. 
We begin with the notion of \textit{task domain}.
\begin{definition}
A task domain $\Omega$ is a state transition system and can be represented as a tuple $\Omega = \langle S, A, \gamma, s_0, S_g \rangle$ where:
\label{def:one}
\begin{itemize}
\item $S$ is a finite set of states;
\item $A$ is a finite set of actions;
\item $\gamma: S \times A \rightarrow S$ such that $s' = \gamma(s, a)$;
\item $s_0 \in S$ is the start state;
\item $S_g \subseteq S$ is the set of goal states.
\end{itemize}
\end{definition}
\noindent Given a task domain, it is possible to define the concept of \textit{task plan}.
\begin{definition}
A task plan for a task domain $\Omega$ is the sequence of actions $a_0, \ldots, a_m$ such that $s_{i+1} = \gamma(s_i, a_i)$, for $i = 0, \ldots, m$ and $s_{m+1}$ \textit{satisfies} $S_g$.
\end{definition}
In general terms, motion planning finds a sequence of collision-free robot configurations from a given start configuration to a desired goal configuration~\cite{latombe1991robot}, which is expressed in terms of (a subset of) desired robot poses.
We define first what we mean by \textit{motion planning domain}. 
\begin{definition}
A motion planning domain is a tuple $M = \langle C, f, q_0, C_g \rangle$ where:
\begin{itemize}
\item $C$ is the configuration space;
\item $f: q \rightarrow \{0, 1\}$, a Boolean function that determines whether a configuration $q$ is in collision ($f=0$) or it is free ($f=1$);
\item $C_g \in C$ is the set of goal configurations.
\end{itemize}
\end{definition}
\noindent It is now possible define the notion of \textit{motion plan}.
\begin{definition}
A motion plan for a motion planning domain $M$ finds a collision-free trajectory in $C$ from $q_0$ to $q_n \in C_g$ such that $f=1$ for $q_0, \ldots, q_n$. 
Alternatively, a motion plan for a motion planning domain $M$ is a function in the form $\tau: [0, 1] \rightarrow C_{free}$ such that $\tau(0) = q_0$ and $\tau(1) \in C_g$, where $C_{free} \subset C$ is the configurations where the robot does not collide with other objects (or itself).
\end{definition}
TAMP combines discrete task planning and continuous motion planning to facilitate an efficient interaction between the two domains.
Below, we formally define what a \textit{task-motion domain} is.
\begin{definition}
A task-motion domain with task domain $\Omega$ and motion planning domain $M$ is a tuple $\Psi = \langle C, \Omega, \phi, \xi, q_0 \rangle$ where:
\begin{itemize}
\item $\phi: S \rightarrow 2^C$, maps states to the configuration space; 
\item $\xi: A \rightarrow 2^C$, maps actions to motion plans.
\end{itemize}
\end{definition}
\noindent Finally, we can define the concept of \textit{task-motion planning problem}.
\begin{definition}
The task-motion planning problem for a task-motion domain $\Psi$ consists in finding a sequence of discrete actions $a_0, \ldots, a_n$ such that $s_{i+1} = \gamma(s_i, a_i)$, $s_{n+1} \in S_g$, and a corresponding sequence of motion plans $\tau_0, \ldots, \tau_n$ such that for $i = 0, \ldots, n$, it holds that:
\begin{itemize}
\item $\tau_i(0) \in \phi(s_i) \ \textrm{and} \ \tau_i(1)  \in \phi(s_{i+1})$;
\item $\tau_{i+1}(0) = \tau_i(1)$;
\item $\tau_i \in \xi(a_i)$.
\end{itemize}
\end{definition}
\noindent This concludes the formalization of the problem we consider in this paper. 

\section{AND/OR Graph Networks}
\label{sec:AOgraph}

In this Section we provide a brief primer about AND/OR graphs, functional to the discussion of the pick-and-place toy example considered in the Introduction, and shown in Fig.~\ref{fig:toy}.
A more thorough exposition can be found in~\cite{karami2020ROMAN}.
An AND/OR graph is a graph with a specific superstructure representing a problem-solving process~\cite{chang1971AI}.
We start with the definition of a \textit{vanilla} AND/OR graph.
\begin{definition}
An AND/OR graph $G$ is a directed graph represented by the tuple $G = \langle N, H \rangle$, where:
\begin{itemize}
\item $N = \{n_1, \ldots, n_{\lvert N \rvert}\}$ is a set of nodes;
\item $H = \{h_1, \ldots, h_{\lvert H \rvert}\}$ is a set of hyper-arcs.
\end{itemize}
\end{definition}
For a given AND/OR graph $G$, each hyper-arc $h$ in $H$ is a many-to-one mapping from a set of child nodes to a parent node, each node being part of $N$. 
With reference to our toy scenario, and just for the sake of explanation, an AND/OR graph $G_{pp} = \langle N_{pp}, H_{pp} \rangle$ for the single pick-and-place scheme is visualized as a sub-graph in Fig.~\ref{fig:pick_place_graph}, whose set of nodes $N_{pp}$ is made up of $5$ nodes, namely $\textsf{object\_clear}$ ($n_1$), $\textsf{gripper\_empty}$ ($n_2$), $\textsf{object\_in\_hand}$ ($n_3$), $\textsf{object\_on\_table}$ ($n_4$), and $\textsf{target\_on\_table}$ ($n_5$), and the set of hyper-arcs $H_{pp}$ comprises $h_1$, $h_2$, and $h_3$.
It is noteworthy that at this stage we ignore the $\textsf{current\_configuration}$ node ($n_6$, in green in the Figure), the arc $h_4$ originating from $n_6$, and the graph extending from the \textsf{object\_placed} node via the hyper-arc $h_5$, both of them described later. 
In a generic AND/OR graph $G$, each node $n \in N$ represents a \textit{high-level}, symbolic description of a fragment of the domain being modeled, that is, the state (such as \textsf{object\_in\_hand} in $N_{pp}$), whereas each hyper-arc $h \in H$ represents one or more actions needed to transit between pairwise sets of symbolic descriptions.   
For instance, and with reference to $G_{pp}$, in order to achieve the state \textsf{object\_in\_hand}, the gripper must be empty and the object be clear from other occluding objects, that is, the symbolic states \textsf{gripper\_empty} and \textsf{object\_clear} must be achieved before any action leading to \textsf{object\_in\_hand} can be considered feasible.
With the term ``achieved'', we mean here that the symbolic state is satisfied, that is, if it is associated with a Boolean truth value, that value should be \textit{true}, whatever the corresponding semantics may be.
Likewise, with the term ``feasible'' we refer to the possibility that, given that prerequisites encoded in its child nodes are met, the actions represented in a hyper-arc can be executed, thus leading to the state encoded in the hyper-arc's parent node.

It is noteworthy that, with a slight abuse of terminology, we can label a node as feasible should it be the co-domain of at least one feasible hyper-arc.
Therefore, the nodes \textsf{object\_clear}, \textsf{gripper\_empty} and \textsf{object\_in\_hand} are akin to the predicates \textsf{clear\ ?x}, \textsf{gripper\_empty}, and \textsf{in\_hand\ ?x}, respectively, mentioned in Section~\ref{sec:intro}. 
Likewise, hyper-arc $h_1 \in H_{pp}$ induces a directed transition from the child nodes \textsf{object\_clear} and \textsf{gripper\_empty}, that is, the prerequisites, to the parent node \textsf{object\_in\_hand}. In other words, \textsf{object\_in\_hand} is the result of a set of actions enabled by the prerequisites, \textsf{object\_in\_hand} and \textsf{gripper\_empty}. In particular, hyper-arc $h_1$ is an example of a transition inducing a logical \textit{AND} relationship between the prerequisites, that is all the child states should be \textit{true} to achieve the parent node/state. 
Similarly, a single parent node can be the co-domain for different prerequisites in alternative. 
The corresponding hyper-arcs are in logical \textit{OR} with respect to the parent node, and at least one child node must be achieved for the parent node to be first feasible and then maybe achieved as well, irrespective of the truth values of other child nodes.
In the formulation adopted for this paper, nodes without any successors (for example, a \textit{root} node) or children (\textit{leaf} nodes) are referred to as \textit{terminal} nodes. 
Terminal nodes are associated with such semantic labels as either \textit{success}, for example, \textsf{target\_on\_table} in the Figure, or \textit{failure}, that is, \textsf{object\_on\_table}. 

Let us again consider the toy table-top scenario in Fig.~\ref{fig:toy}. 
If the number of objects to be re-arranged were known, then a comprehensive AND/OR graph $G_t$ could be composed using an \textit{appropriate} sequence of multiple pick-and-place schemes, all of them in the form of $G_{pp}$, that is, the above mentioned nodes and the related hyper-arcs. 
Such an AND/OR graph $G_t$ would be a sequence of multiple instances of $G_{pp}$, each one for a different object to re-arrange via a single pick-and-place scheme, until the target object was grasped.
However, the number of object re-arrangements is workspace-dependent, and can not be considered completely and reliably known before the sequence of actions unfolds. 
Therefore, it would not be possible to reliably encode an AND/OR graph structured as $G_t$ as a \textit{static} structure. 
However, we make the following observation: 
the structure of the AND/OR graph $G_{pp}$ defined for the \textit{single}  pick-and-place instance, that is, the states (nodes of $G_{pp}$) and actions (hyper-arcs of $G_{pp}$), remains the same irrespective of the number of object re-arrangements. 
Thus, we can envision a graph $G_t$ initially coinciding with $G_{pp}$, but able to recursively \textit{expand} with further instances of $G_{pp}$ (each one operating on a different object) until the target object is clear and therefore it can be grasped. 
Such a recursive expansion would happen, should the result of the upstream instance of $G_{pp}$ lead to a failure node, in our case \textsf{object\_on\_table}. 
Nevertheless, each iteration of $G_{pp}$ yields a new workspace configuration, in terms of object arrangements, because it implies the displacement of one object as a result.
In order to explicitly represent this aspect during the graph expansion, we need to \textit{augment} $G_{pp}$ with a node representing the current workspace configuration, that is, \textsf{current\_configuration}, and with the related hyper-arcs $h_4$ and $h_5$, as shown in Fig.~\ref{fig:pick_place_graph}.
Such a node represents, for each expansion of $G_{pp}$, the equivalent of a new \textit{initial} description of the workspace. 
Augmented AND/OR graphs are defined as follows.
\begin{definition}
For an AND/OR graph $G=\langle N, H \rangle$, an augmented AND/OR graph $G^a$ is a directed graph represented by the tuple $G^a = \langle N^a, H^a \rangle$ where:
\begin{itemize}
\item $N^a = N \cup \{n^a\}$, with $n^a$ being the augmented node;
\item $H^a = H \cup \{h^a_1, \ldots, h^a_A\}$, where all $h^a_i$ are the augmented hyper-arcs.
\end{itemize}
\end{definition}
The augmented node $n^a$ is referred to as the root node of the augmented graph $G^a$.
With respect to $G_{pp}$, we need to consider an augmented AND/OR graph $G^a_{pp}$ where $n^a$ coincides with \textsf{current\_configuration}, which can be reached via the hyper-arc $h_5$.
Furthermore, each augmented hyper-arc $h^a_i$, namely $h_4$ and $h_5$ in our example, induces a mapping between the \textsf{current\_configuration} node and a node of the related, non augmented graph $G_{pp}$.

It is now manifest that $G_t$ is iteratively expanded as many instances of $G^a_{pp}$ as needed, one for each object to re-arrange.
In particular, in the scenario we consider, we envisage an overall AND/OR graph recursive expansion as follows.
The most trivial case corresponds to being able to successfully grasp the target object without any object re-arrangements.
In this case $G_t$ corresponds to a single instance of $G^a_{pp}$, which we refer to as $G^a_{pp,1}$.
However, most often due to the degree of clutter in the robot workspace, a feasible motion plan to reach the target object does not exist, and therefore objects need to be re-arranged to obtain a successful plan. 
This fact corresponds to a failure in $G^a_{pp,1}$, that is, its node \textsf{object\_on\_table} is satisfied.
Once a cluttering object to be removed is identified (more on this in the next Section), the model entailed by an AND/OR graph in the form of $G^a_{pp}$ is again used (therefore leading to the expansion of $G_t$ to a sequence of $G^a_{pp,1}$ and $G^a_{pp,2}$). 
This is repeated until in a certain iteration of the expansion process the object considered in \textsf{object\_clear} corresponds to the target object, which leads the graph to a success node, that is \textsf{target\_on\_table}.  

In the scenario discussed here, we have in $G^a_{pp}$ only one success node and one failure node.
It is possible to postulate that multiple, alternative success and failure nodes may exist within one AND/OR graph.
Depending on their semantics, this fact may lead to multiply connected AND/OR graphs, and therefore to an AND/OR graph network.
This notion is better formalized as follows.

\begin{definition}
An AND/OR graph network $\Gamma$ is a directed graph $\Gamma = \langle \mathcal{G}, T \rangle$ where:
\begin{itemize}
\item $\mathcal{G} = \{G^a_1, \ldots, G^a_{\abs{\mathcal{G}}}\}$ is a set of augmented AND/OR graphs in the form $G^a_i$;
\item $T= \{t_1, \ldots, t_{\abs{\mathcal{G}}-1}\}$ is a set of transitions such that $G^a_{i+1} = t_i(G^a_i), \ 1 \leq i \leq \abs{\mathcal{G}}-1$.
\end{itemize}
\end{definition}

\begin{figure}[]
\centering
\includegraphics[width=1.0\textwidth]{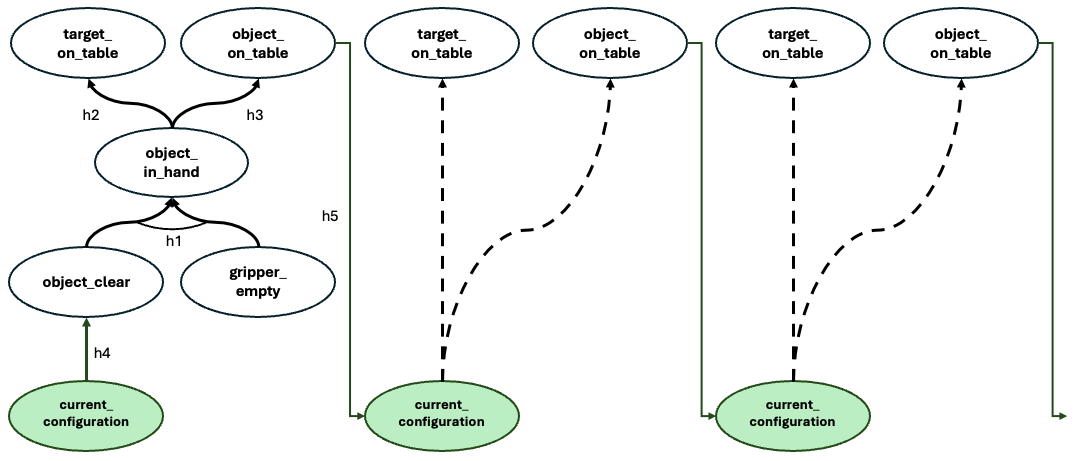}
\caption{A possible AND/OR graph network for the pick-and-place scenario shown in Fig.~\ref{fig:toy}.}
\label{fig:pick_place_graph}
\end{figure}
\noindent While $\abs{\mathcal{G}}$ is the total number of graphs in the network, it can also be interpreted as the depth of the network. 
Furthermore, it is also worth noting that a transition $t_i$ is defined only if the upstream AND/OR graph $G^a_i$ is \textit{unsolved}, that is, if it terminates in a failure node. 
When this happens, $t_i$ leads $G^a_i$ to a new augmented AND/OR graph $G^a_{i+1}$ such that $G^a_{i+1} = t_i(G^a_i)$, and in our scenario this implies that the root node of $G^a_{i+1}$ encodes a changed workspace configuration, leading to an AND/OR graph network as shown in Fig.~\ref{fig:pick_place_graph}. 
We can now specify when an AND/OR graph network is solved.\\

\noindent \textbf{Proposition 1}
An AND/OR graph network $\Gamma = (\mathcal{G}, T)$ is said to be solved at depth $d$ when $G^a_d$ is solved.\\

\noindent Following Proposition 1, an AND/OR graph network with underlying augmented AND/OR graphs $G^a_i$ is expanded till a $G^a_i$ is solved. 
It is necessary to better develop a few insights related to the computational complexity of the whole expansion process.
In Section~\ref{sec:intro} we have discussed the computational complexity for the toy example shown in Fig.~\ref{fig:toy} as if it were modeled using PDDL-like predicates. 
We now analyze the complexity aspects when the same scenario is modeled using our AND/OR graph network. 
Assuming that sufficient time is allotted for the motion planner, in the worst case all the $5$ objects need to be re-arranged leading to $5\times5$ states\footnote{This is computed as $5$ nodes for each object as seen in Fig.~\ref{fig:pick_place_graph} and $5$ iterations due to $5$ objects.}, as opposed to $2^{13}$ using PDDL (see Section~\ref{sec:intro}). 
However, it is noteworthy that, in real-world environments, motion executions may fail due to, for example, actuation or grasping errors. 
This situation does not imply a new arrangement of objects, since the object to be re-arranged remains the same. 
In such a case a new graph $G^a_{i+1}$ would be expanded, retaining the same workspace configuration as in $G^a_{i}$. 
As a consequence, the number of expansions $m$ may be greater than the number of objects (here $m>5$) and hence the time complexity is thus $O(5m) \approx O(m)$. 
In case of PDDL-based planning, the shortest path in the solution space takes $O(n \log n)$ time, where $n = 2^{13}$ is the number of states. 
In general, for an AND/OR graph network with each graph consisting of $n$ nodes, the worst-case temporal complexity is only $O(nm)$.
Similarly, for an AND/OR graph with $n$ nodes, a storage of only $O(n)$ nodes is required.

\section{\textsc{TMP-IDAN}}
\label{sec:approach}

\subsection{Single-robot \textsc{TMP-IDAN}}
\label{subsec:srTAMP}

\subsubsection{System's Architecture}
\label{sub:architecture}

An overview of the architecture supporting \textsc{TMP-IDAN} is given in Fig.~\ref{fig:arch1}. 
In general, the approach requires a robot architecture with the following capabilities:
(i) a \textit{perception layer} to make sense of objects in the robot workspace, 
(ii) a \textit{planning} and \textit{reasoning layer} to synthesize a plan of discrete, symbolic actions, and  
(iii) a \textit{motion planning layer} for executing object manipulation actions while avoiding potential collisions. 
In the current version of the system's architecture, all perception capabilities are encapsulated in a single module, which is called \textit{Scene Perception} (\textsf{SP}). 
This module provides the \textit{Knowledge Base} (\textsf{KB}) module with the information about the current status of the workspace, that is, the location of objects on a table in front of the robot, as well as the robot's configuration. 
\textit{Knowledge Base} acts as a bridge between the perception layer and the planning and reasoning layer, which is made up of three modules, namely the \textit{TMP Interface} (\textsf{TMPI}), the \textit{Motion Planner} module (\textsf{MP}), and the \textit{Task Planner} (\textsf{TP}) module. 
\textit{TMP Interface} receives discrete, symbolic commands from \textit{Task Planner}, maps them to actual geometric features and quantities in the workspace via information encoded in the \textit{Knowledge Base}, and drives the behavior of \textit{Motion Planner}. 
The latter retrieves information regarding the workspace and the robot itself from the \textit{Knowledge Base} to ground action commands with robot motions.
Furthermore, it provides an acknowledgment to the \textit{Task Planner} module upon the execution of a command by the robot via the \textit{TMP Interface} module.
To this aim, \textit{Motion Planner} plans the outcome of robot motions before their actual execution via an internal in-the-loop simulation. 

\begin{figure*}[]
\centering
\includegraphics[width=0.8\textwidth]{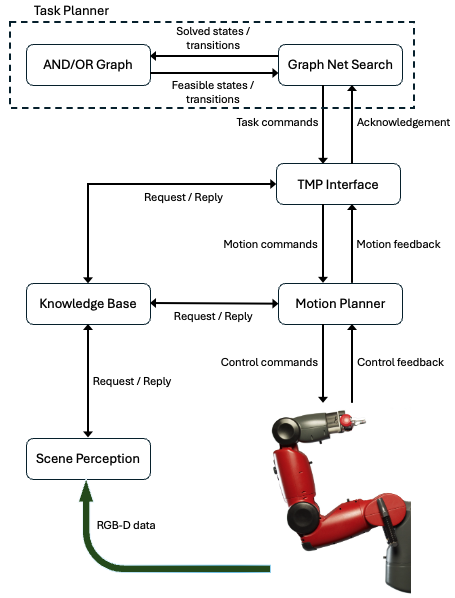}
\caption{System's architecture of the \textsc{TMP-IDAN} framework.}
\label{fig:arch1}
\end{figure*}

\textit{Task Planner} encapsulates two sub-modules, namely the \textit{Augmented AND/OR Graph} (\textsf{AOG}) module and the related AND/OR \textit{Graph Net Search} (\textsf{GNS}) module. 
In a general sense, the \textit{Task Planner} module is in charge of decision making and adaptation of the ongoing parallel hyper-arcs, where by ``parallel'' we mean different sets of feasible actions given the current state. 
To do so, the \textit{Augmented AND/OR Graph} module provides a set of achieved transitions between the states to \textit{Graph Net Search}, and receives the set of allowed states and state transitions to follow, along with the associated costs. 
Once an action is carried out, \textit{Graph Net Search} receives an acknowledgment from the action execution layer, and updates the AND/OR graph data structure accordingly. 
Likewise, the \textit{Knowledge Base} stores all relevant information about the workspace to support the execution of future actions.
The \textit{Task Planner} subscribes to relevant information from the \textit{Knowledge Base} through the \textit{TMP Interface}, for example, whether an object is moved, and this drives the \textit{Graph Net Search} forward.

In order to implement the motion planning layer, we recur to MoveIt~\cite{sucan2013moveit}, and in particular to the variant of RRT~\cite{kuffner2000ICRA} provided therein. 
\textit{Motion Planner} is called to 
(i) achieve the re-arrangement of each object (if any) identified by the \textit{Task Planner} module, and 
(ii) grasp the target object. 
\textit{Motion Planner} requests from the \textit{Knowledge Base} the current state of the workspace, carries out in-the-loop simulations varying motions via RRT and allocations, performs task allocation should more than one robot be available accordingly, and then executes the identified motion plan.

\begin{algorithm}[]
\scriptsize
\caption{\textsc{TMP-IDAN}}
\label{algo}
\Input{\textsf{AOG}, \textsf{TP}, \textsf{TMPI}, \textsf{MP}, \textsf{KB}, \textsf{SP}}
$\textit{\textsf{reached(}gc\textsf{)}} \gets false$
$i \gets 0$
$g \gets G_i$
\While{$\neg \textsf{reached(}gc\textsf{)}$}
{
    $G^a_i \gets (\textit{\textsf{AOG}} \rightarrow NAG(G_i))$
    
    $fts \gets (\textit{\textsf{GNS}} \rightarrow Req(\textit{\textsf{AOG}} \rightarrow GFTS(G^a_i)))$ 
    
    \eIf{$fts = \emptyset$}
    {
        $\textbf{exit}$
    }
    {
        $ots \gets (\textit{\textsf{GNS}} \rightarrow FNOTS(fts))$
        
        \ForAll{$(a,r) \in ots$}
        {
            $nc \gets (\textit{\textsf{GNS}} \rightarrow Req(\textit{\textsf{TMPI}}, \textit{[a,r]}))$
            
            $\textit{\textsf{GNS}} \rightarrow Req(\textit{\textsf{AOG}} \rightarrow Update(nc))$ 
        }
        
        $i \gets i + 1$
    }

    \When{$* \rightarrow Req($\textsf{TMPI}, [a,r]$)$}
    {
        \ForEach{$a \in A$}
        {
            \ForEach{$r \in R$}
            {
                $cs \gets (\textit{\textsf{TMPI}} \rightarrow Req(\textit{\textsf{KB}}))$
                    
                $nc \gets (\textit{\textsf{TMPI}} \rightarrow Req(\textit{\textsf{MP}}, a, cs))$
                        
                \eIf{$nc = gc$}
                {
                    $\textit{\textsf{reached(gc)}} \gets true$
                }
                {
                    $\textit{\textsf{reached(gc)}} \gets false$
                }
 
                $\textbf{return}\; nc$
            }
        }
    }
    
    \When{$* \rightarrow Req($\textit{\textsf{KB}}$)$}
    {
        $ns \gets (\textit{\textsf{KB}} \rightarrow Req(\textit{\textsf{SP}}))$
        
        $\textbf{return}\; ns$
    }
    
    \When{$* \rightarrow Req(\textsf{SP})$}
    {
        $rs \gets (\textit{\textsf{SP}} \rightarrow Read(\textsf{camera\_data}))$

        $ns \gets (\textit{\textsf{SP}} \rightarrow CDP(\textsf{rs}))$
        
        $\textbf{return}\; ns$
    }
    
    \When{$* \rightarrow Req(\textsf{MP}, a, cs)$}
    {
     \ForAll{$actions, agents \in a$}
        {
            $P \gets (\textit{\textsf{MP}} \rightarrow Simulate(a, cs))$            
                }

        $p \gets (\textit{\textsf{MP}} \rightarrow FBT(a, P))$
        
        $d \gets (\textit{\textsf{MP}} \rightarrow Execute({p}))$
        
        \eIf{$d = false$}
        {
            $P=P/p$
            
                \eIf{$P = \emptyset$}
            {
            $nc = NULL$
                       }
              {
            Goto $line$ $35$  
                             }
            }
        {
            $nc \gets \textit{\textsf{GCC}(p)}$
          }
                  
            $\textbf{return}\; nc$
    }
}
\end{algorithm}

\subsubsection{Task Planning with AND/OR Graph Networks}

The AND/OR graph representation provides a framework for the planning of action sequences~\cite{sanderson1988TAES}. 
The feasibility of the motions associated with the high-level, discrete actions is then checked using a suitable motion planner~\cite{karami2020ROMAN} within an in-the-loop simulation process carried out by the \textit{Motion Planner} module. 
Moreover, an AND/OR graph inherently requires fewer nodes than the corresponding complete state transition graph, thus reducing the search complexity of the AND/OR space~\cite{sanderson1988TAES}. 
Yet, as discussed above, such a representation requires that the number of object re-arrangements to retrieve a target object from clutter is known ahead of time. 
As a consequence, this representation seems incompatible with our goal, as we do not know before-hand the number of objects to be re-arranged. 
In order to overcome this limitation, we leverage AND/OR graph networks as discussed in Section~\ref{sec:AOgraph}, wherein an augmented AND/OR graph grows while planning and execution unfold until the target grasping action can be executed. 
It is noteworthy that, depending on the level of abstraction and how we choose to model the actions, it is possible to construct different AND/OR graphs and thereby different, corresponding graph networks for the same table-top scenario (or, for the sake of the argument, any other real-world scenario). 
For example, in this paper, for the single-robot scenarios using a Baxter robot, which we will discuss later, we have designed an AND/OR graph $G_0$, which is shown in Fig.~\ref{fig:combined} (right).
As we discussed above, this basic structure can be recursively grown by replicating it a sufficient number of times.
This originates the combined structure formed by $G_0$ and $G_1$, which constitutes the corresponding AND/OR graph network with depth 2, shown in the same Figure. 
In $G_0$, a node labeled \textsf{INIT\#0}, which represents the initial workspace configuration, is added thereby originating the augmented AND/OR graph $G_0^a$. 
The graph $G_0^a$ encodes the fact that if a feasible grasping trajectory to the target object existed, then the \textsf{target\_picked} state could be reached by performing the actions encoded in the hyper-arc between \textsf{INIT\#0} and \textsf{target\_picked}, which would be a \textit{success} node. 
Otherwise, an object should be identified to be either pushed or removed (this fact can be encoded in the \textit{Knowledge Base} module and it is based on the relative size of the object and the robot grippers). 
Let us consider the case that such target grasping trajectory does not exist, and that the states \textsf{closest\_object\_target\_grasped} and later \textsf{object\_in\_storage} are achieved.
The latter is considered a \textit{failure} node, that is, a node not leading to the overall goal, and therefore exhausting $G_0^a$. 
This leads to the need to add a new graph $G_1^a$, which corresponds to a new augmented graph with the same states and actions as $G_0^a$ but with a different initial workspace configuration due to the traversal along $G_0^a$, as encoded via the augmented node \textsf{INIT\#1}. 
This growth process iteratively repeats itself until a graph $G_{n'}^a$ is solved, which represents an AND/OR graph network $\Gamma$ of depth $n'$. 
The augmented AND/OR graphs $G^a_0, \ldots, G^a_{n'}$ are thus iteratively deepened to obtain an AND/OR graph network whose depth $n'$ is task-dependent and not known in advance. 
We note here that in case a task in $G_i^a$ fails, for example due to motion, actuation or grasping errors, $G_{i+1}^a$ (initialized with the current workspace configuration) is grown and in this way our approach shows some form of robustness with respect to execution failures. 

Workflow~\ref{algo} describes the overall planning procedure, and can be considered as a description of the data flow in the overall robot architecture. 
The whole procedure is essentially an iterative process running until the goal configuration $gc$ (corresponding to the target grasp) is \textsf{reached}. 
Assuming that initially the AND/OR graph coincides with $G_0$ (but what follows can be also applied to any subsequent $G_i$, and therefore in this description we use the index $i$), the \textit{Augmented AND/OR Graph} module augments it as $G^a_i$ by calling the function \textit{NewAugmentedGraph} ($NAG$), which adds the initial workspace configuration (line~3). 
The \textit{Graph Net Search} module then requires \textsf{AOG} to determine the set of feasible transitions and states \textsf{fts} associated with $G^a_i$ via the function \textit{GetFeasibleTransitionsStates} ($GFTS$, line 3). 
If such a set is not empty, then (lines 7-11) the main loop occurs.
First, the set of optimal transitions and states $ots$ is generated by \textit{Graph Net Search} via a call to the function \textit{FindNextOptimalTransitionsStates} ($FNOTS$). Please note that, in general, $ots$ can be a set of different pairs of transitions and states, since more than one pair can have the same cost. 
Then, for each pair of optimal transitions and states, the new configuration $nc$ is obtained through \textsf{TMPI} (line 12).
\textsf{TMPI} calls the $Req(\textsf{MP}, a, cs)$ subroutine (lines 29-42), which is responsible for computing $nc$. 
We observe here that a transition (that is, a hyper-arc) encodes one or more actions needed to transit between states. 
Therefore, for this set of actions and agents (each arm of the robot is considered a separate agent) $Req(\textsf{MP}, a, cs)$ simulates the motion plan. \textit{FindBestTransition} ($FBT$) then extracts the motion plan corresponding to the minimum cost and this plan is then executed (line~36). 
If execution fails, then the next plan is selected (line~38) for execution.
If all plans $P$ are exhausted without a successful execution, then a $NULL$ value is returned.
The $GetCurrentConfiguration$ ($GCC$) returns the new configuration (or $NULL$) resulting from the execution.
If all the $nc$ returned (line 12) are $NULL$ due to execution failures, then a re-try is attempted by augmenting the current graph with a new graph with identical workspace configuration. 
The workspace configuration is identical since all $nc$ are $NULL$ and hence its update results in the same workspace configuration for the augmented graph. 


\subsubsection{Probabilistic Completeness}

We now prove the probabilistic completeness of \textsc{TMP-IDAN}.
\begin{lemma}
\textsc{TMP-IDAN} is probabilistically complete.
\end{lemma}
\begin{proof}
\textsc{TMP-IDAN} is made up of two planning procedures, namely for task and motion planning.
We consider the two separately.

For task planning, we employ an AND/OR graph network \(\Gamma = (\mathcal{G}, T)\).
By design, each sub-graph \(G^a_i\) in \(\mathcal{G}\) terminates at either a success node or a failure node, indicating the completeness of \(G^a_i\). 
Let us define a depth limit \(l\) for the expansion of \(\Gamma\), such that \(\Gamma\) is expanded to at most depth \(l\).
We analyze two cases:
\begin{enumerate}
    \item 
    If there exists a solution within the depth \(l\), then there exists some \(j \leq l\) such that \(G^a_j\) is solved, and consequently, \(\Gamma\) is solved.
    \item
    If no solution is found within depth \(l\), then \(\Gamma\) terminates at \(G^a_l\), and we report a non-existence of solutions.
    This non-existence could be due to either an actual non-existence of such a solution, or the insufficiency of the chosen depth \(l\).
\end{enumerate}
To address the case in which the chosen depth is insufficient, we increment \(l\) and consider the limit as \(l\) approaches infinity.
Formally, as \(l \to \infty\), the probability \(P\) of \(\Gamma\) being solved, if a solution exists, asymptotically approaches \(1\), that is:
\begin{equation}
\lim_{l \to \infty} P(solved(\Gamma)) = 1.
\end{equation}

For motion planning, we observe that we utilize the RRT motion planner~\cite{kuffner2000ICRA}, which is known to be probabilistically complete~\cite{karaman2011IJRR}.
This implies that, given sufficient time, the probability \(P\) of finding a feasible motion $\tau$, if one exists, such that $\tau: [0, 1] \rightarrow C_{free}$, approaches \(1\), that is:
\begin{equation}
\lim_{t \to \infty} P(\tau: [0, 1] \rightarrow C_{free}) = 1.
\end{equation}

Therefore, given that both the task planner based on an AND/OR graph network and the RRT motion planner are designed to be probabilistically complete, we can conclude that as the depth limit \(l\) for the task planner and the time for the RRT planner approach infinity, the probability of finding a solution, if one exists, approaches \(1\).
This establishes the probabilistic completeness of \textsc{TMP-IDAN}.
\end{proof}

\begin{figure}[]
\centering
\includegraphics[width = 1.0\textwidth]{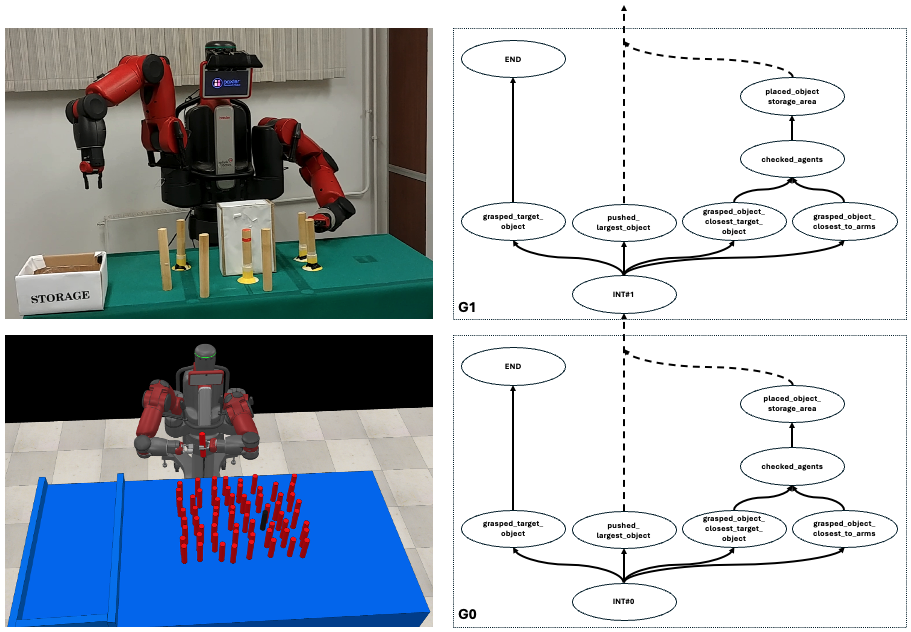}
\caption{\textsc{TMP-IDAN} in a real-world use (top-left), and in simulation (bottom-left).
An AND/OR graph network of depth 2 for the corresponding cluttered scenario (right), where we do not label hyper-arcs.}
\label{fig:combined}
\end{figure}
        
\subsection{Multi-robot \textsc{TMP-IDAN}}

\begin{figure}[]
\centering
\subfloat[]{\includegraphics[width=5.3cm, height=5.3cm]{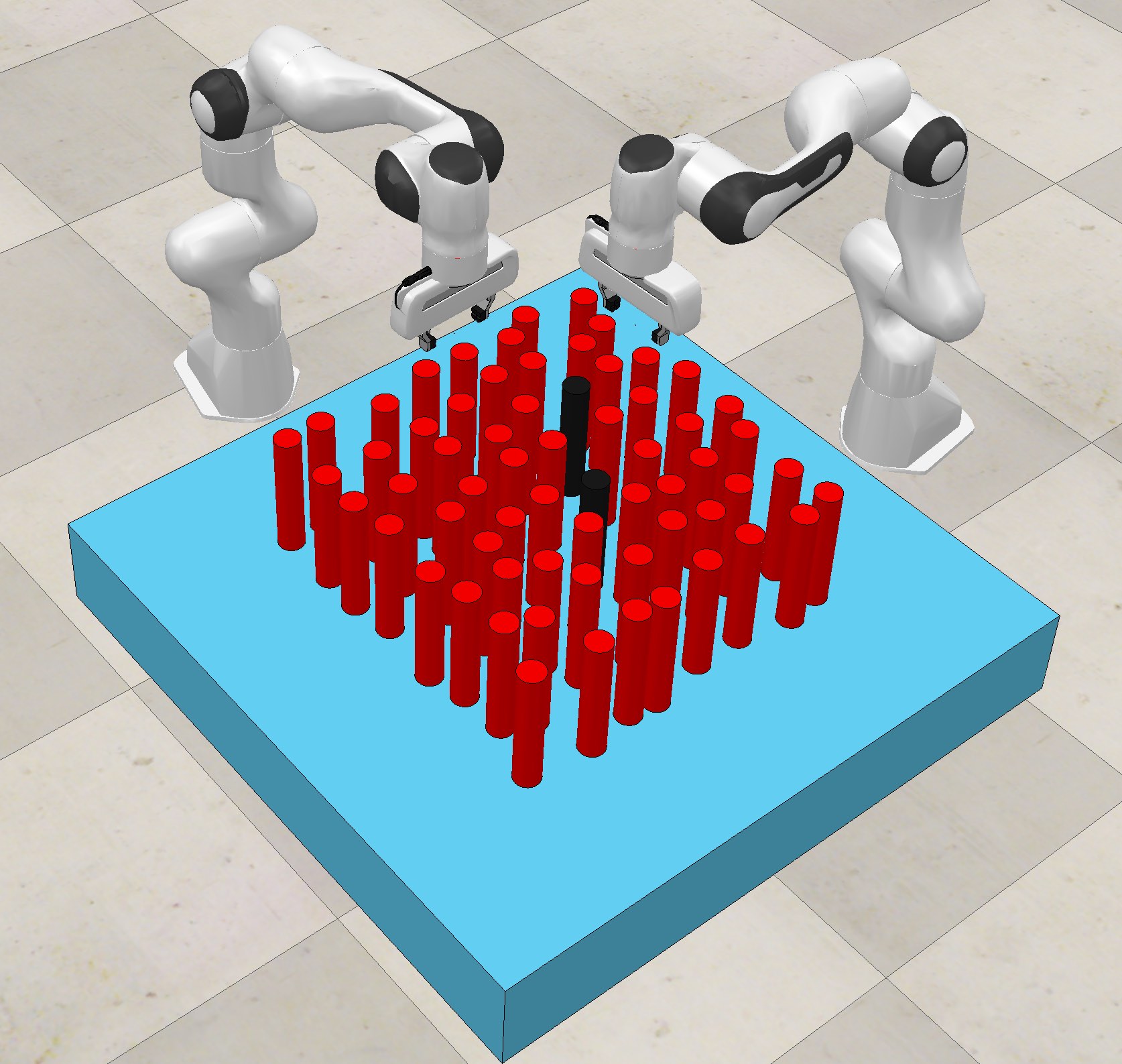}\label{fig:t0}}\hspace{0.1cm}
\subfloat[]{\includegraphics[width=5.3cm, height=5.3cm]{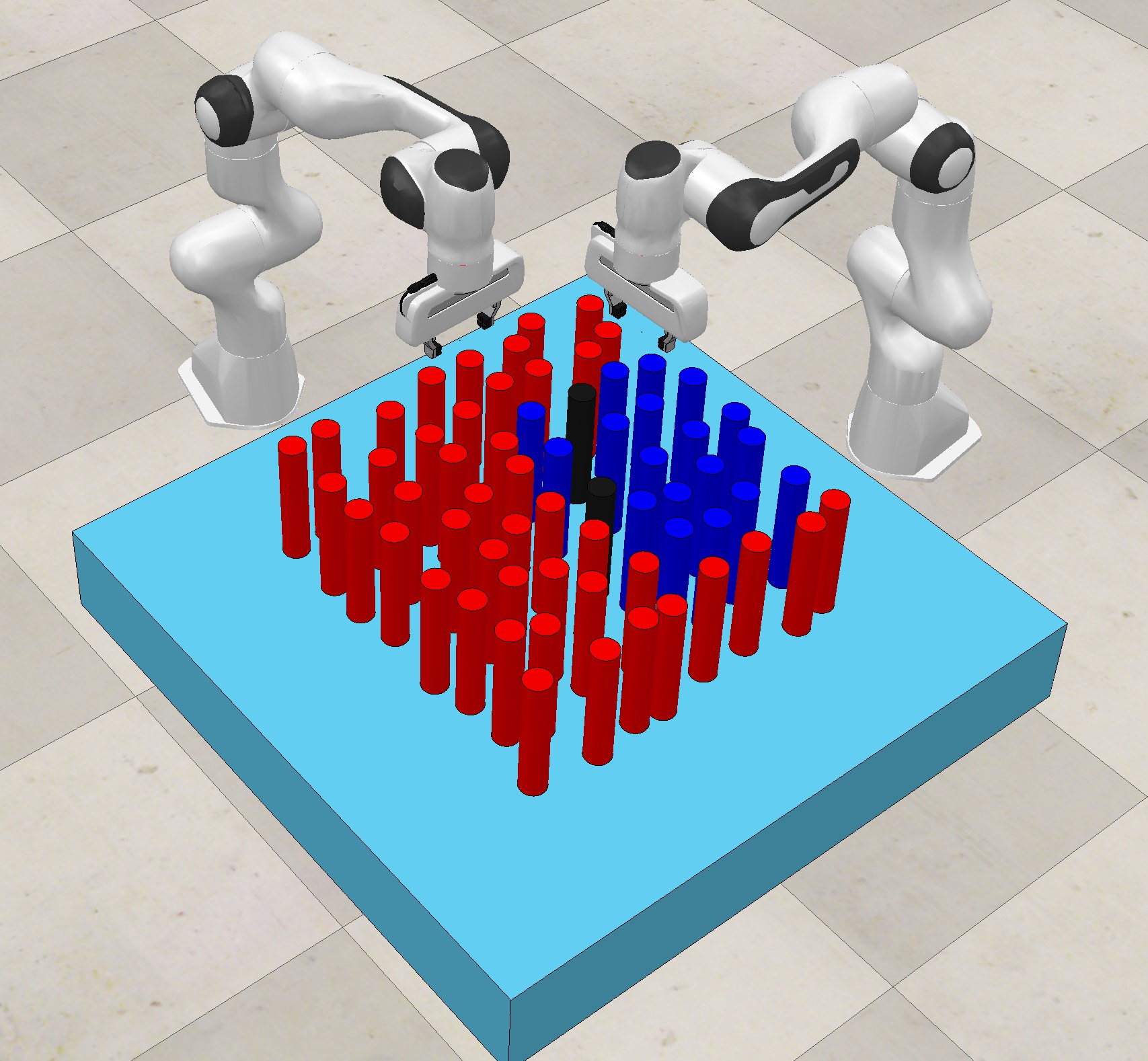}\label{fig:t1}}\\
\subfloat[]{\includegraphics[width=5.3cm, height=5.3cm]{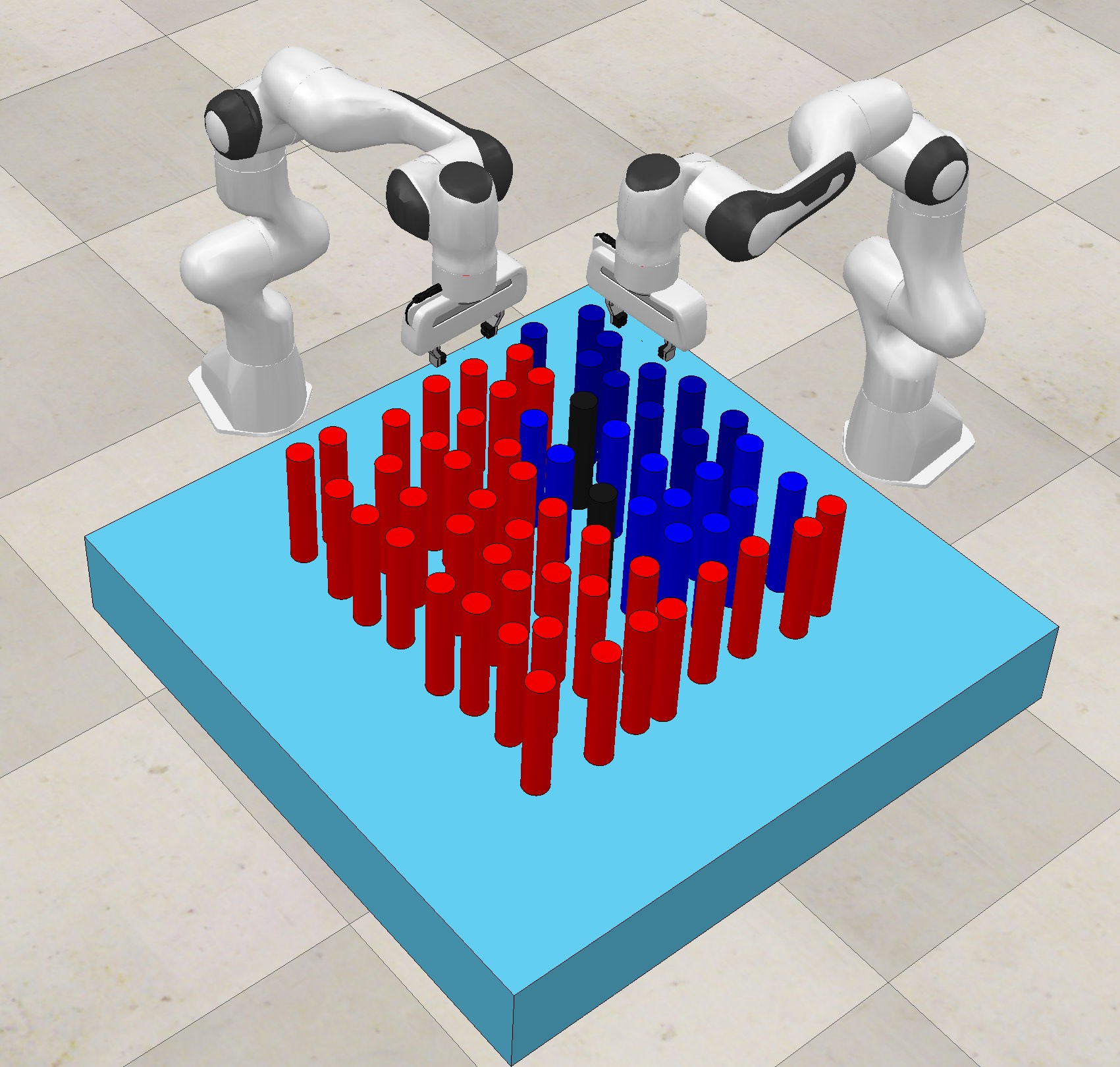}\label{fig:t2}}\hspace{0.1cm}
\subfloat[]{\includegraphics[width=5.3cm, height=5.3cm]{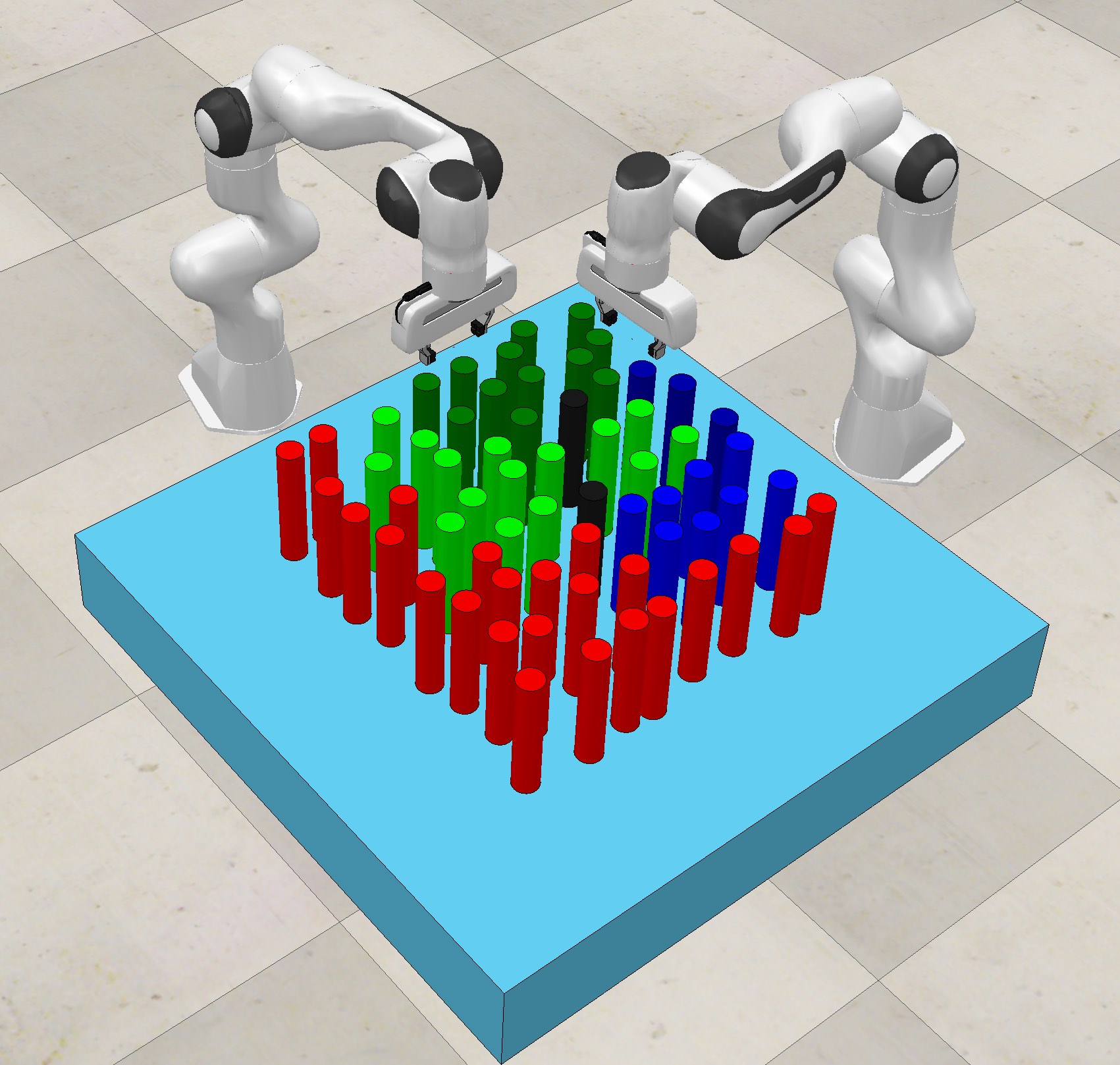}\label{fig:t3}}
\caption{
Illustration of the obstacles identification method:
(a) a cluttered table-top scenario with two robots, R1 (right) and R2 (left), and the target objects in black;
(b) a valid grasping angle range is computed by discretizing a fixed grasping angle range from $-\frac{\pi}{2}$ to $\frac{\pi}{2}$; the objects falling within the grasping angle range of R1 (one target) are shown in blue;
(c) blue objects within the grasping angle range of R1, considering both  targets;
(d) objects within the grasping range of both R1 (blue) and R2 (green), considering both the targets for each robot.}
\label{fig:task}
\end{figure}
We leverage the concepts detailed in the previous Sections, and present a framework for multi-robot TAMP.
As a case in point, we consider a scenario in which objects in a cluttered table-top setting must be re-arranged with the aim of reaching target objects as shown in Fig.~\ref{fig:toy2}.
It is noteworthy, though, that we constrain the robot motions to side grasps only, with the purpose of making the scenario more challenging. 
To this end, task allocation among available robots must be carried out, which is followed by a TAMP method to carry out the allocated tasks.

We begin by describing a simple, heuristic approach providing a rough estimate of which objects must be re-arranged to be able to reach the target object. 
The heuristic approach allows us to define a combined utility function for the multi-robot system to perform task allocation.
It is noteworthy, though, that other approaches are equally legitimate, and multi-robot TMP-IDAN does not rely on this specific approach for its overall behavior.
Allocated tasks are then carried out using \textsc{TMP-IDAN}.

\subsubsection{Identification of the Obstacles}

An illustration of the obstacles identification method can be seen in Figure~\ref{fig:task}.
First, the method finds different, feasible motion plans to reach the target object, each one corresponding to different grasping angles, while purposely ignoring all the obstacles in the workspace. 
This is done by discretizing the set of graspable angles, that is, in the range $\left[-\frac{\pi}{2}, \frac{\pi}{2}\right]$, assuming that the axis is located at the center of the target object.
The available motion plans are then ranked on the basis of their grasping angle.
The maximum and the minimum grasping angles, namely $\alpha$ and $\beta$, are then obtained. 
Two imaginary line segments, starting from the center of the target object, and directed towards the robot, with the angles $\alpha$ and $\beta$, are defined.
The line segments ends where there are no more obstacles along their paths. 
Their end-points are then joined to form a triangle, which is then enlarged on all the three sides by the radius of the bounding volume sphere of the robot's end-effector. 
The objects within the constructed triangle are those to be re-arranged to enable grasping the target object. 

We note that the method we describe to identify the objects to be re-arranged is an heuristic leading to an approximate result.
In fact, the actual set of objects to re-arrange would depend on other such factors as the robot's degrees of freedom, the size of the links and the end-effector, or the \textit{degree} of clutter.
The minimum constraint removal problem~\cite{hauser2014IJRR, thomas2023IAS} identifies the smallest set of obstacles to be re-arranged, while the minimum constraint displacement problem~\cite{hauser2013RSS, thomas2022IAS, thomas2023ICRA} computes the minimum amount by which obstacles must be re-arranged to enable a feasible grasp. 
However, these methods are NP-hard in general.
As it will be discussed later, for all practical purposes we are interested only in an approximate measure so as to perform multi-robot task allocation. 
\subsubsection{Task Allocation}
Let $R$ be the number of available robots, and let $T$ be the number of tasks to be allocated, that is, we have $T$ target objects to be grasped, such that $T \geq R$. 
Task allocation is performed offline.
We assume that each task, that is, picking up a target object from clutter, is performed by a single robot, and that each robot is able to execute only one task at a time. 
Our task allocation strategy falls under the Single-Task, Single-Robot, Time-extended Assignment (ST-SR-TA) taxonomy introduced by Gerkey and Matari{\'c}~\cite{gerkey2004IJRR}, since the multi-robot system contains more tasks than robots. 
In order to allocate tasks to robots, we define $U_{r_it_j}$ as the utility function for a robot $r_i \in R$ executing a task $t_j \in T$.
In this work, utility is defined to be inversely proportional to the number of object re-arrangements required to grasp the target object. 
To determine such a measure, we first (randomly) select a target object $t_1$, and then for each robot $r_i$ we apply the obstacles selection algorithm described in the previous Section. 
For each $r_i$, the heuristic strategy returns a set of objects to be re-arranged in order to reach $t_1$. 
Let us denote this set by $O_{r_it_1}$ (and by $O_{r_it_j}$ for the $j$th task). 
The robot $r_k$ whose set $O_{r_kt_1}$ is of minimum cardinality (that is, , maximum utility) is then allocated task $t_1$.
For the next target $t_2$ this step is repeated. 
We note here that $O_{r_kt_2}$ is computed offline and returns the number of objects to be re-arranged by robot $r_k$ to execute task $t_2$. 
However, it may be the case that some objects appear in both $O_{r_kt_1}$ and $O_{r_kt_2}$, that is, $O_{r_kt_1} \cap O_{r_kt_2} \neq \{\emptyset\}$. 
Since each robot executes one task at a time, $r_k$ can execute $t_2$ only after having performed $t_1$. 
Thus, during the execution phase, $r_k$ may have already removed the common objects while executing $t_1$, and therefore these objects may be ignored to avoid \textit{intra-robot} double counting while computing the utility $U_{r_kt_2}$ offline. 
Once each robot is allocated a task, the set of remaining tasks $T'$ is completed only after the execution of the assigned tasks. 
For the remaining $T'$ tasks, the \textit{inter-robot} or robot-robot double counting must be considered. 
Reasoning in a similar manner for intra-robot double counting, the set $O_{r_it_j}$ restricted to $T \setminus T' < j \leq T$ for the remaining $T'$ tasks may have common objects with respect to $O_{r_it_j}$ restricted to $1 < j \leq T'$ of the assigned tasks. 
Thus, the total number of objects to be re-arranged for robot $r_i$ to execute task $t_j$ is
\begin{equation}
O_{r_it_j}^c = \abs{O_{r_it_j}} - \sum_{k} \abs{O_{r_it_jt_{k}}} - \sum_{k} \sum_{l}\abs{O_{r_ir_kt_jt_l}},
\end{equation}
\noindent where $\abs{\cdot}$ denotes the cardinality of a set, 
\begin{equation}
    O_{r_it_jt_{k}} =
    \begin{cases}
    O_{r_it_j}\cap O_{r_it_k} \ &\text{if $r_i$ allotted $t_k$ previously}, \\
    0 \ &\text{otherwise},
    \end{cases}
\end{equation}
\noindent and
\begin{equation}
    O_{r_ir_kt_jt_l} =
    \begin{cases}
    O_{r_it_j}\cap O_{r_kt_l} \ &\text{if $r_k$ allotted $t_l$ previously}, \\
    0 \ &\text{otherwise},
    \end{cases}
\end{equation}
\noindent with the terms $\abs{O_{r_it_jt_{k}}}$ and $\abs{O_{r_ir_kt_jt_l}}$ modeling the intra-robot and inter-robot double counting, respectively. 
We therefore have the following utility function:
\begin{equation}
    U_{r_it_j} = \frac{1}{1+O_{r_it_j}^c}. 
    \label{eq:utility}
\end{equation}
The maximum utility of $U_{r_it_j} = 1$ is therefore achieved when no object re-arrangement is required to execute task $t_j$, that is, $O_{r_it_j}^c = 0$. 
Using the taxonomy in~\cite{korsah2013IJRR}, we thus have In-schedule Dependencies (ID) -- the effective utility of an agent for a task depends on what other tasks that agent is performing as well as Cross-schedule Dependencies (XD) -- the effective utility of an agent for a task depends not only on its own task but also on the tasks of other agents. 
We now define the combined utility of the multi-robot system, which consists of maximizing
\begin{equation}
    \sum_{i \in R} \sum_{j \in T} U_{r_it_j}x_{r_it_j}\
    \label{eq:comb_utility}
\end{equation}
\begin{equation}
\begin{split}
& \textrm{such that} \, \sum_{i \in R} x_{r_it_j} = 1.\\
\end{split}
\label{eq:comb_utility1}
\end{equation}
\noindent where $x_{r_it_j} \in \{0,1\}$ is an integer and a robot $r_i$ is assigned to the task $t_j$ only when $x_{r_it_j} = 1$. From~\eqref{eq:utility} and~\eqref{eq:comb_utility} we see that the robot with the minimum number of object re-arrangements for a given task is thus assigned the maximum utility. 
In case of a tie, we select the robot which has not been allocated any task. 
If all the robots with the same utility have been allocated tasks already, or if none has been allotted, then a robot is selected randomly. 
\subsubsection{Task Decomposition}
Each manipulation task is decomposed into a set of sub-tasks which correspond to pick-and-place tasks, that is, re-arrangement of the objects that hinder the grasping the target object. 
As seen above, the number of sub-tasks for a given task is not known beforehand. 
Moreover, \textit{multiple decomposability}~\cite{zlot2006phd} is possible since in general a cluttered space can be re-arranged in different ways. 
We seek a decomposition minimizing the number of sub-tasks for the multi-robot system. 
This can be achieved during task allocation since the utility function is computed based on the obstacles selection method. 
In this work, we consider \textit{complex task decomposition}~\cite{zlot2006phd} -- a multiply decomposable task for which there exists at least one decomposition that is a set of multi-robot allocatable sub-tasks. 
Though in this work we ignore the multi-robot allocatability property, this can be incorporated trivially. 
For example, let us consider the case where two robots have the same utility to perform a task $t_j$. 
In this case, the sub-tasks can be equally divided to achieve multi-robot allocatability or one robot may be selected randomly (or depending on previous allocations) to perform the entire task.
\subsubsection{The Multi-robot Task-Motion Planning Loop}
The overall system's architecture of our multi-robot \textsc{TMP-IDAN} method is similar to the single-robot architecture shown in Fig.~\ref{fig:arch1} with the only difference being the \textit{Network Search} module.
Since we have multiple robots and their respective configuration updates, we have multiple \textit{Network Search} modules corresponding to each robot, that is, $n$ such modules for an $n$-robot system.
Once the tasks have been allocated, \textit{Task Planner} selects the abstract actions whose geometric execution feasibility is checked by \textit{Motion Planner}. 
The \textit{Task Planner} layer consists of the \textit{AND/OR Graphs} module -- the initial augmented AND/OR graphs for the robots, and the \textit{Network Search} module -- the search procedure iterating the initial augmented graphs.
As discussed previously, the initial augmented AND/OR graph consists of the task-level actions for each robot augmented with the corresponding initial workspace configuration. 
The \textit{AND/OR Graphs} module provides a set of achievable transitions between the states to \textit{Network Search}, and receives the set of allowed states and transitions as the graph is expanded. 
\textit{Task Planner} then associates each state or state transition with an ordered set of actions in accordance with the geometric locations of objects and the robots. 
The \textit{Knowledge Base} module stores the information regarding the current workspace configuration and uses that to augment the graphs with the aim of facilitating \textit{Network Search}.

As seen before, the \textit{TMP Interface} module acts as a bridge between the task planning and the motion planning layers. 
It receives action commands from \textit{Task Planner}, converts them to their geometric values (for example, a grasping command requires various geometric values such as the target pose or the robot base pose), and passes them on to \textit{Motion Planner} to check motion feasibility. 
To this end, the module retrieves information regarding both the workspace and robots from \textit{Knowledge Base}. 
If an action is found to be feasible, it is then sent for execution. 
Upon execution, \textit{Task Planner} receives an acknowledgment regarding action completion and the \textit{Knowledge Base} is updated accordingly.  
\section{Experimental Results}
\label{sec:results}
In this Section, we analyze \textsc{TMP-IDAN} in the context of both single and multi-robot scenarios.
All experiments are conducted on a workstation equipped with an Intel(R) core i7-8700@3.2 GHz $\times$ 12 CPU's and 16 GB of RAM. 
Simulations have been implemented on CoppeliaSim~\cite{coppeliaSim}.
The software architecture was developed using C++ and Python under ROS Kinetic.
For motion planning, we use MoveIt!\cite{sucan2013moveit}, in particular the RRT planner from OMPL~\cite{sucan2012RAM}, and we place a temporal upper bound of $1$ second.
\subsection{Single-robot \textsc{TMP-IDAN}}
\subsubsection{Retrieving a Target Object from a Cluttered Workspace}

The goal of this class of experiments is two-fold:
(i) characterize the computational performance of TMP-IDAN with respect to a number of controlled parameters, specifically in the case of a \textit{large task space}, with the aim of corroborating our intuition that TMP-IDAN may behave almost linearly;
(ii) evaluate the adaptability of \textsc{TMP-IDAN} in view of \textit{unfeasible actions}, that is, possible failures specifically in the motion planning and execution phases which hinder the planning and execution of discrete actions. We note here that in this work, we do not consider the aspect of \textit{task/motion trade-off}.

To this aim, we have set-up a cluttered table-top scenario, wherein target objects in the form of cylinders and cuboids are to be reached and grasped, as shown in Fig.~\ref{fig:combined} (top-left for the real scenario, and bottom-left for the simulation). 
Experiments are performed using a dual-arm Baxter manipulator from Rethink Robotics.
The robot is equipped with its standard grippers, whereas an RGB-D camera, mounted on its \textit{head} and pointing downward, is used to acquire images for object detection and recognition.
Re-arranged objects are moved to a storage area located close to the robot's right arm, and in case the object to re-arrange is grasped by the left arm, it is first handed to the right arm and then placed in the storage area. For the real scenario, seven cylinders are placed on a table and the robot is required to pick up the object (target) with red tape. The white box has to be pushed since it occludes the target. However, few objects around the white box need to be removed before pushing. The experiment was conducted eight times and the overall planning and execution times are shown in the first four rows of Table~\ref{table:real_scenario}. As seen in the last two rows Table~\ref{table:real_scenario}, the graph was iteratively deepened about 15 times in each experiment with around 130 motion planning attempts. The increased depth and motion planning attempts are due to the failure of the motion planner in finding feasible plans.
\begin{table}
\centering
\begin{tabular}{cc}\hline
    \hline
        AND/OR Graph & 0.06$\pm$0.01 [s] \\
        Graph Net Search &0.09$\pm$0.02 [s]  \\
        Motion Planner &6.20$\pm$0.70 [s] \\
                 d & 15.00$\pm$3.00 \\
         MP attempts& 130.00$\pm$23.00 \\ \hline
      \end{tabular}
     \caption{Quantitative results for the real robot scenario.}
\label{table:real_scenario}
\end{table}

We performed simulations wherein the number of objects ranges from $4$ to $64$.
For a given number of objects, three simulation runs are carried out randomly sampling the target location. 
Fig.~\ref{fig:combined} (right) shows a fragment of the employed AND/OR graph network. 
For the generic AND/OR graph $G_i$, replicated here twice for the first two iterations, several parallel paths are possible. 
The state $grasped\_target\_object$ is reached when the robot is able to reach and grasp the target object without obstacles, and therefore it leads to a successful graph execution, that is, the $END$ state.
Otherwise, if the target object cannot be reached, it is necessary to re-arrange the cluttered workspace by removing obstructing objects.
If the object selected for removal is too big, it can be pushed, which leads to state $pushed\_largest\_object$, and subsequently to a new graph iteration ($G_1$ in the Figure).
Otherwise, if it is possible to grasp an occluding object close to the target one, this can be done, and it leads to state $grasped\_object\_closest\_target\_object$; if that is not possible, an object close to the current robot end-effectors is selected, which leads to the state $grasped\_object\_closest\_to\_arms$. 
Either the latter states lead eventually to the selected object be picked up and placed in the storage area, which leads to the state $placed\_object\_storage\_area$ and therefore to a new graph iteration.
All robot actions are encoded in the hyper-arcs, which are all in \textit{OR}.
Given the fact that we use a dual-arm manipulator, we assume to have two \textit{agents}, namely the Baxter's right and left arms.
Before the execution of a given action realizing a hyper-arc, agents are grounded to actions on the basis, in this case, of geometrical considerations, that is, a picking motion is allocated to the robot arm closest to the object to pick.

\begin{table*}[t]
\centering
\scalebox{0.7}{
\begin{tabular}{ c c c c c c} 
\hline
\textit{Objects}    & Avg. \textit{d}    & \textit{TP} [s]   & \textit{MP} [s]   & \textit{MP attempts}  & \textit{Objects to re-arrange} \\
\hline
\hline
4                   & 1.67          & 0.02             & 1.04             & 15.66                       & 1.33 \\
8                   & 7.33          & 0.07             & 4.57             & 50.66                       & 4.00 \\
15                  & 14.33         & 0.17             & 20.89            & 177.00                      & 7.50 \\
20                  & 57.00         & 0.42             & 61.17            & 400.00                      & 18.00 \\
30                  & 19.66         & 0.19             & 28.59            & 159.66                      & 9.00 \\
42                  & 72.00         & 0.47             & 44.15            & 384.50                      & 18.50 \\
49                  & 26.50         & 0.34             & 24.27            & 201.00                      & 10.00 \\
64                  & 76.00         & 0.66             & 102.58           & 693.00                      & 29.00 \\
\hline
\end{tabular}}
\caption{Trend of the average network depth \textit{d}, the average total task planning time \textit{TP}, the average total motion planning time \textit{MP}, the number of average motion planning attempts and the average number of objects to be re-arranged, as the number of objects varies from $4$ to $64$ in the cluttered workspace domain.}
\label{tab:table0}
\end{table*}

Considering these simulations, Table~\ref{tab:table0} reports the average network depth \textit{d}, the average total task planning time \textit{TP}, the average total motion planning time \textit{MP}, the average number of motion planning attempts, and the average number of objects to be re-arranged.
The trend in task planning time \textit{TP} with the increase of network depth \textit{d} seems to corroborate our intuition on the temporal complexity (see Section~\ref{sec:AOgraph}), and it is almost linear with respect to $d$. 
However, motion planning failures due to grasping failures or occlusion lead to a frequent re-planning, therefore explaining the large number of motion planning attempts. 

\begin{table}[t]
\centering
\scalebox{0.7}{
\begin{tabular}{lcc} 
\hline
\textit{Computation time}   & \textit{Average} [s]    & \textit{Standard Deviation} [s] \\ 
\hline
\hline
AND/OR Graph                & 0.01                                & 0.01 \\ 
Graph Net Search            & 0.01                                & 0.01 \\
Motion Planner              & 0.80                                & 0.22 \\
Motion Executor             & 4.75                                & 2.02 \\
\hline
\end{tabular}}
\caption{Computation times for selected modules of \textsc{TMP-IDAN}, and their corresponding standard deviations, for the cluttered workspace domain with $4$ objects.}
\label{tab:comp1}
\end{table}

Table~\ref{tab:comp1} shows the average execution times for a selection of the modules introduced in Section~\ref{sec:approach}.
It is worth mentioning that \textit{Motion Executor} is a \textit{virtual} module representing the time it takes for the robot to execute the planned motion trajectories. 
As it can be observed, the \textit{discrete} part of the planning process, namely the one implemented by the concurrent activity of the modules \textit{AND/OR Graph} and \textit{Graph Net Search}, exhibits low average running times and a small standard deviation.
The \textit{Motion Planner} module is relatively more computationally demanding, with an increased standard deviation due to the need to generate different, feasible trajectories.
However, both the discrete and continuous planning processes are fast if compared with the execution of movements, as one would expect.

\begin{figure}[H]
\centering
\subfloat[]{\includegraphics[width=7cm]{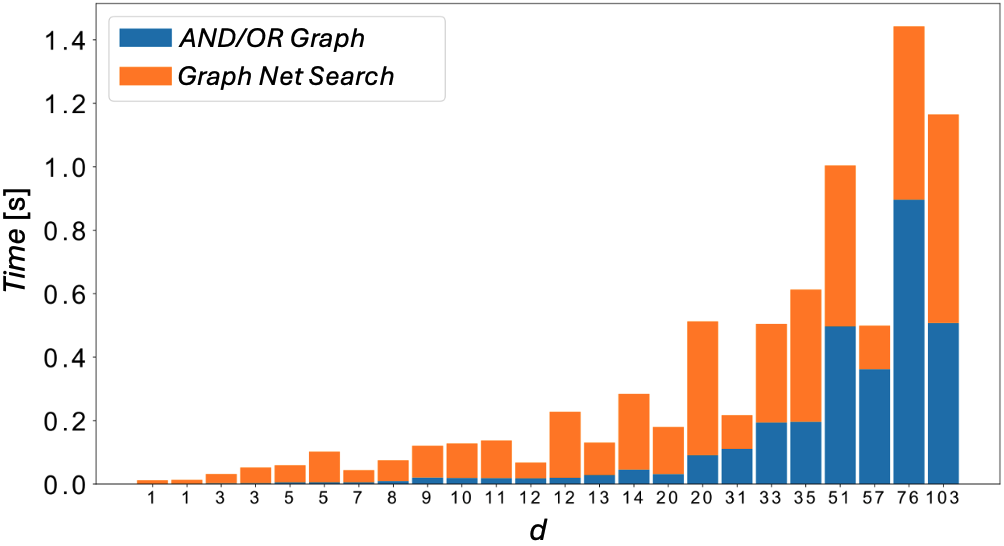}\label{fig:compa}}\\
\subfloat[]{\includegraphics[width=7cm]{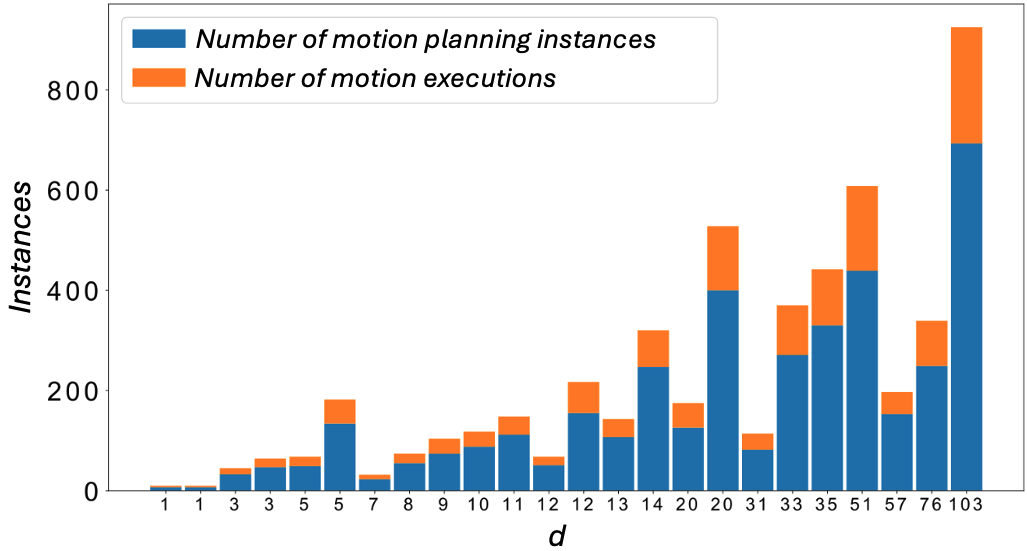}\label{fig:compb}}\\
\subfloat[]{\includegraphics[width=7cm]{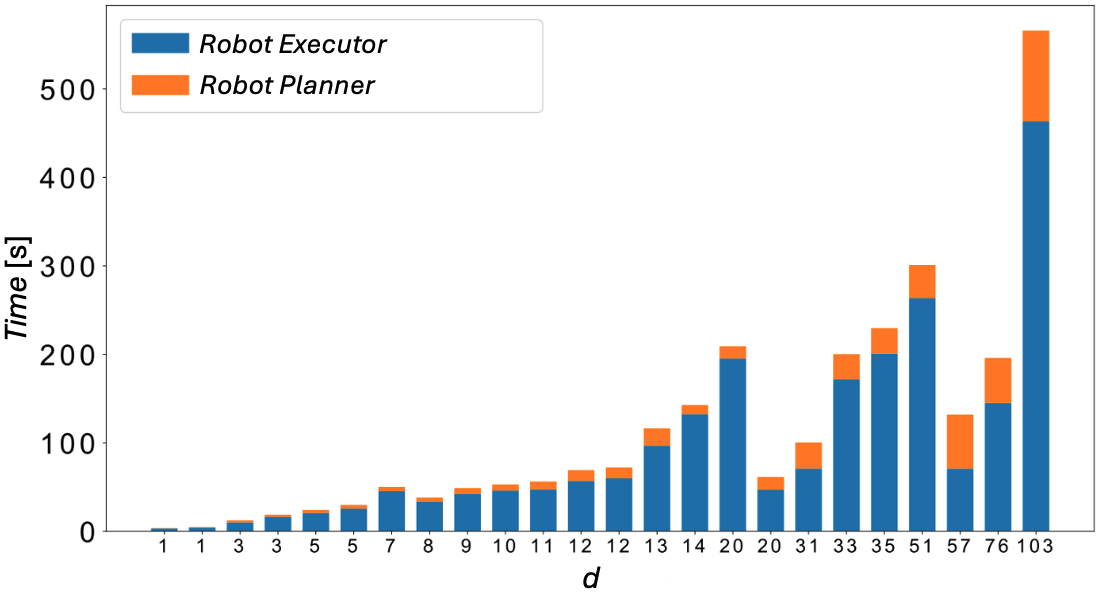}\label{fig:compc}}
\caption{Various trends associated with an increasing number of depth \textit{d} for the cluttered workspace domain: 
(a) for small values of \textit{d} the computational time of \textit{Graph Net Search} dominates the one of \textit{AND/OR Graph}, whereas as \textit{d} increases the two computational times become comparable;
(b) the number of motion planning instances is very high compared with actual executions, since \textit{Motion Planner} tries various motion strategies before them being executed;
(c) the increase in \textit{d} has a minor effect on the overall planning and execution time, as execution takes significantly more time over planning.}
\label{fig:comp}
\end{figure}

In addition to the computational complexity analysis reported in Table~\ref{tab:table0}, Fig.~\ref{fig:comp} shows different plots with an increasing network depth $d$. 
Fig.~\ref{fig:comp}(a) shows the total task planning time with $d$.
As seen above, planning times are almost linear with $d$. 
However, slight deviations are observed, for example, the time for $d=76$ is greater than the time for $d=103$. 
This is due to the fact that in many cases, due to motion planning failure, a new graph is expanded before reaching the terminal node.
Therefore, for $d=76$, more nodes are traversed when compared to the case with $d=103$. 
Fig.~\ref{fig:comp}(b) reports the number of motion planning attempts and the total motion executions. 
This values depend on the degree of clutter, and thereby on the number of required object re-arrangements. 
This trend is more clearly observed in the last two columns of Table~\ref{tab:table0} where an increase in the number of motion planning attempts can be seen with an increased number of object re-arrangements. 
Finally, in Fig.~\ref{fig:comp}(c), total motion and execution times are shown, where it is evident that motion execution always take longer time than planning. 

\begin{figure}[H]
\centering
\subfloat[]{\includegraphics[width=0.48\textwidth,height=6cm]{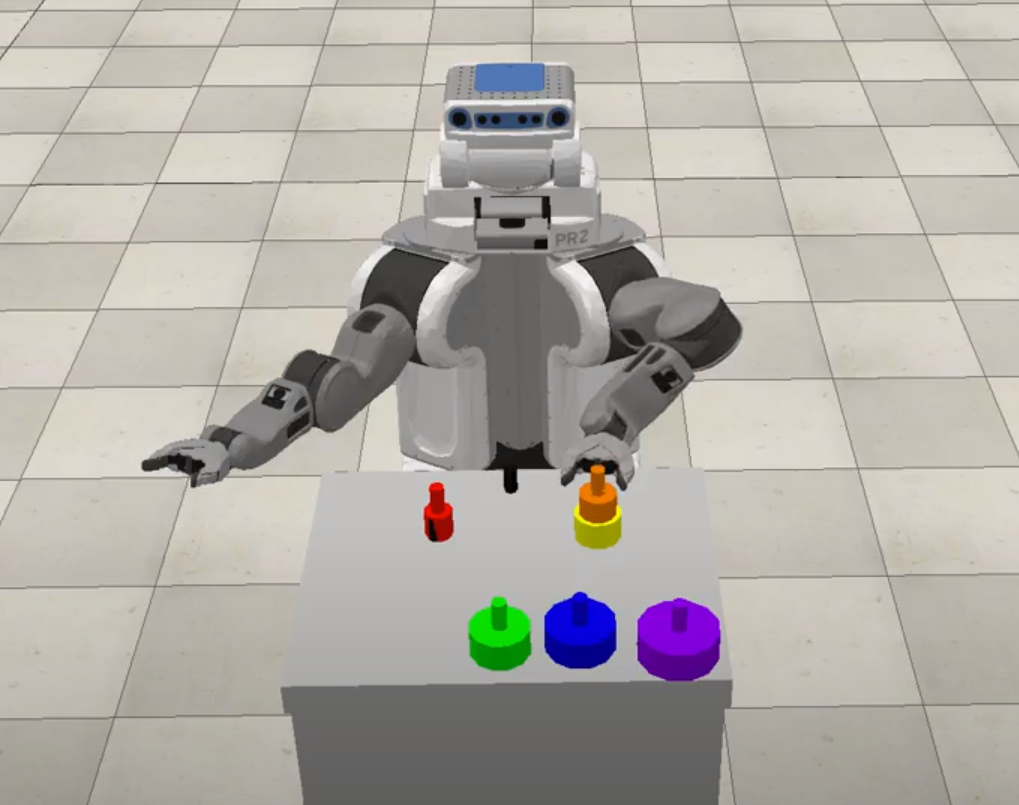}\label{fig:hanoi0}}\hfill
\subfloat[]{\includegraphics[width=0.48\textwidth,height=6cm]{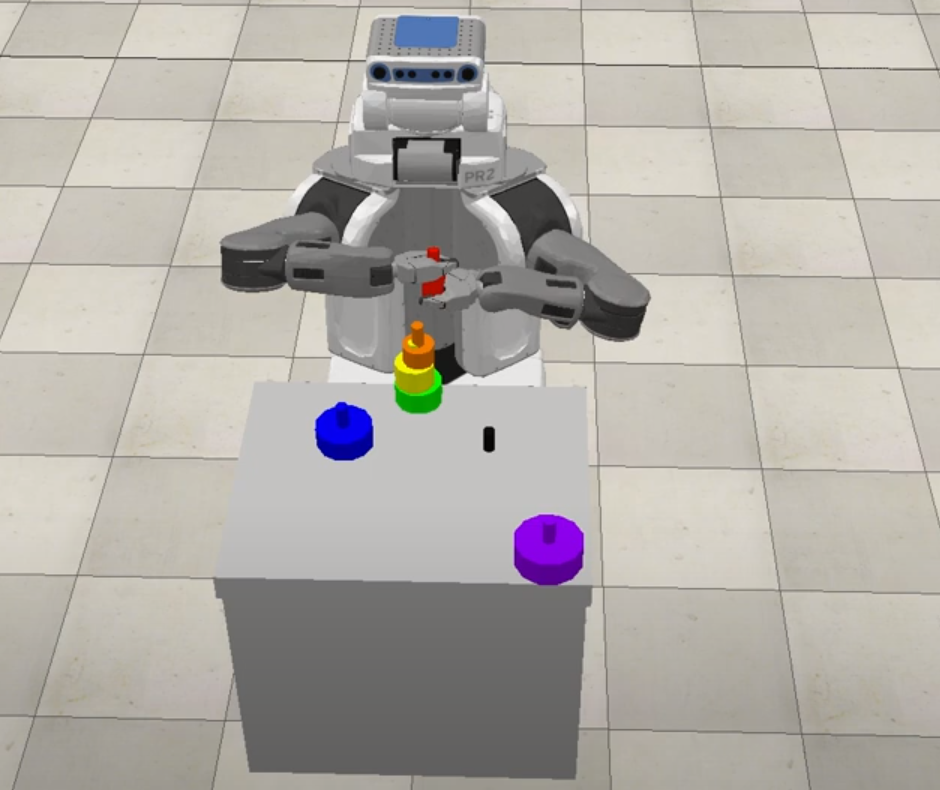}\label{fig:hanoi1}}\\
\subfloat[]{\includegraphics[width=0.48\textwidth,height=6cm]{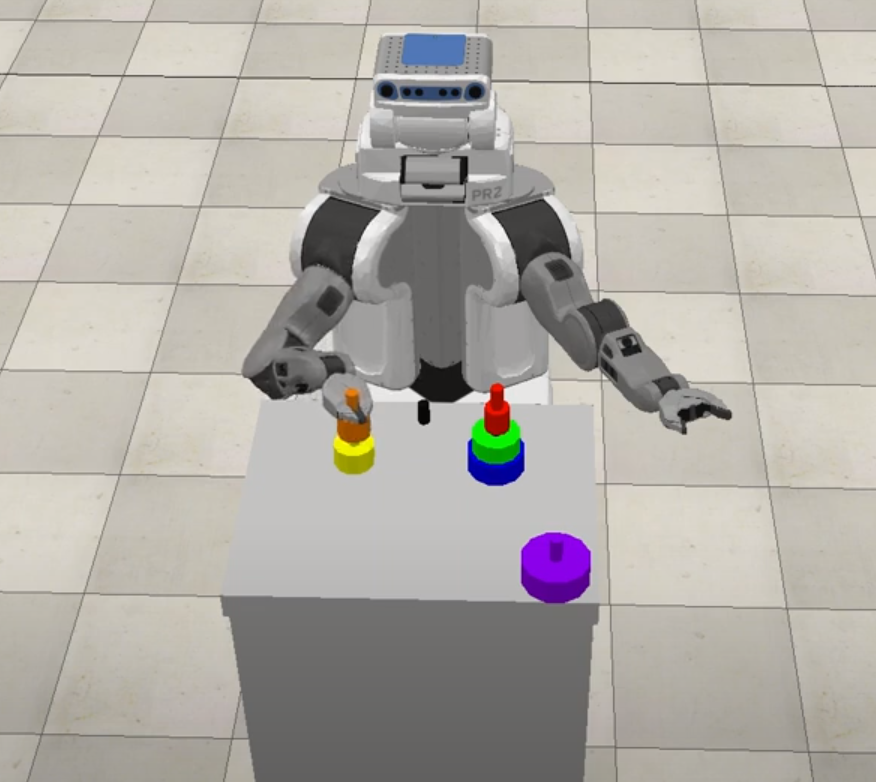}\label{fig:hanoi2}}\hfill
\subfloat[]{\includegraphics[width=0.48\textwidth,height=6cm]{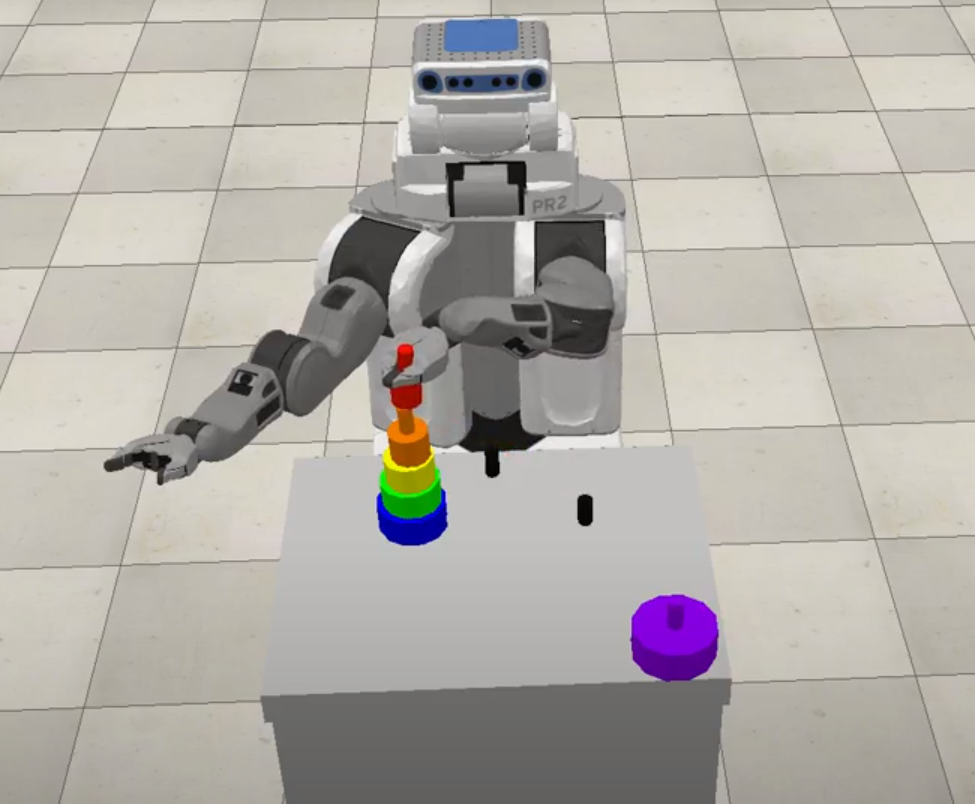}\label{fig:hanoi3}}
\caption{A PR2 robot solves a Tower of Hanoi problem in the simulation environment: (a) PR2's left arm picks a disk up to place it onto an intermediate peg;
(b) the left arm handovers a disk to the right arm due to its kinematic limits in reaching the destination peg;
(c) the right arm places a disk onto the left peg;
(d) five disks are placed to their destination peg successfully.}
\label{fig:tamp_hanoi}
\end{figure}
\subsubsection{Solving the Tower of Hanoi Game}
The goal of this class of experiments is to analyze the capabilities of \textsc{TMP-IDAN} in terms of unfeasible actions and large task spaces in a domain with different characteristics with respect to the one in which the robot must retrieve objects from a cluttered workspace, as discussed in the previous Section.

We simulate a scenario that is a variant of the classic Tower of Hanoi problem, wherein disks stacked on one rod in order of decreasing size must be moved to another rod, as shown in Fig.~\ref{fig:tamp_hanoi}.
It is important to note that, for each rod, disks can be placed only on bigger disks, and only one disk can be moved at each step. 
Experiments are performed with a dual-arm PR2 mobile manipulator from Willow Garage.
Also in this case the robot is equipped with standard grippers and an RGB-D camera for object detection and recognition.
The PR2 base is fixed in front of a table, and three rods are therein arranged in a triangular shape, such that stacking a disk on a rod may prevent (temporarily) picking or placing another disk on any other rod.
It must be observed that, given $n$ disks, the optimal solution involves $2^n -1$ disk arrangements, assuming that each action is feasible. 
We performed simulations wherein the number of disks is between $3$ and $6$.
For a given number of disks, three simulations are carried out. 
One possible AND/OR graph for the Tower of Hanoi problem is depicted in Fig.~\ref{fig:hanoi_graph}, which comprises 21 nodes and 33 hyper-arcs.
In the generic graph $G_i$, which models the movement of a single disk, after ensuring that the state of the workspace has been acquired and processed, thus reaching the $updated$ state, it is evaluated whether it is possible to move a disk currently located at a certain rod to one of the other two rods, and one specific move is selected, which lead to one of the states in the form $checked\_ij$, where $i$ and $j$ represent, respectively, the \textit{from} and \textit{to} rods.
On the basis of geometric considerations, it is decided which rod to pick a disk from, which leads to one of the states in the form $picked\_i$, where $i$ refers to the \textit{from} rod.
If the disk can be placed on the \textit{to} rod by the same robot arm, the graph reaches one of the states $placed\_to\_rod\_j$, with $j$ different from $i$.
Otherwise, if the selected \textit{to} rod is not reachable by the robot arm holding the disk, an handover occurs between the two arms, which is modelled by the sequence of states $handed\_over\_left$, $handed\_over\_right$, and $handed\_over$, followed by the actual placement of the disk on the selected rod.
Once a placement is done, a check is made whether the current configuration is the final one, which leads to the state $staged$: if so, via the state $all\_done$ the graph is solved; otherwise, a new iteration is done via the state $not\_done$.
As in the previous domain, all robot actions are encoded within the hyper-arcs, which are all in \textit{OR}.
Since PR2 is a dual-arm manipulator, also in this case we assume to have two \textit{agents}, namely PR2's right and left arms.
The agents are grounded to actions on the basis of geometrical considerations, that is, which is the best arm to pick up a given disk is computed at run-time before the action is executed.

Table~\ref{tab:hanoi_av} reports, for a number of disks between $3$ and $6$, the average network depth \textit{d}, the average total task planning time \textit{TP}, the average motion planning time with the right arm \textit{Right MP}, the average number of motion planning attempts with the right arm, the average motion planning time with the left arm \textit{Left MP}, and the average number of motion planning attempts with the left arm, along with the respective standard deviations.
In general, the results in the Table confirm our initial intuitions, and align with what we observed previously. 
It is worth mentioning that as long as the number of objects increases from $3$ to $6$, the number of motion planning attempts for the two arms increases significantly.
It can also be observed that the attempts for the right and left arms are not well-balanced, which is due to the specific configuration of the workspace and the task.
An analysis of the trend in depth $d$ highlights an interesting fact about the iterative nature of TMP-IDAN.
As we anticipated, if $n$ is the number of disks in the problem instance, it is possible to compute the number of needed pick-and-place actions to optimally solve the problem, that is, $2^n-1$.
Ideally, this number should be equal to $d$, since the single graph $G_i$ represents a single pick-and-place task. 
However, we observe that actual values of $d$ approximately double the theoretical value, and tend to diverge from it for bigger values of $n$.
Considering only the averages, we have $d=16.67$ in place of $d=7$ when $n=3$, $d=31$ in place of $d=15$ when $n=4$, $d=61$ in place of $d=31$ when $n=5$, and $d=63$ in view of $d=140$ when $n=6$. 
As noted before, this is mainly due to the fact that some of the symbolic actions are unfeasible at the geometrical level. 
For many actions, due to obstacles (since the disks are thick, stacking a single disk on a rod may prevent picking or placing another disk on any other rod) and kinematic limits of the robot arms, the robot may not be able to move a disk from its place to the destination. 
This requires the use of intermediate pegs, which results in the need to iteratively add AND/OR graphs in the form of $G_i$. 
Furthermore, motion planning failures due to
possible actuation errors or grasping failures may lead to further re-planning, which would further add to the network depth. 
\begin{figure}[H]
    \centering
    \includegraphics[width=0.9\textwidth]{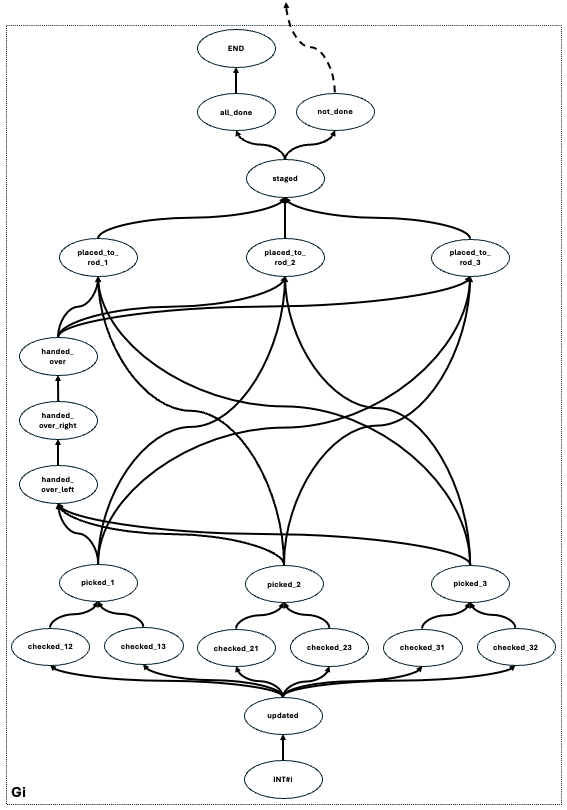}
    \caption{An AND/OR graph for the Tower of Hanoi problem; all hyper-arcs are in OR, and are not labeled.}
    \label{fig:hanoi_graph}
\end{figure}

\begin{table}[t]
\centering
\scalebox{0.7}{
\begin{tabular}{ccccccc} 
\toprule
\textit{Objects}    & \textit{d}        & \textit{TP} [s]   & \textit{Right MP} [s] & \textit{Right MP attempts}    & \textit{Left MP} [s]  & \textit{Left MP attempts} \\
\hline
\hline
3                   & 16.67$\pm$3.51    & 0.78$\pm$0.01     & 105.67$\pm$18.10      & 76.06$\pm$0.74                & 111.37$\pm$12.25      & 88.10$\pm$4.43 \\
4                   & 31.00$\pm$8.18       & 3.27$\pm$0.37     & 245.64$\pm$47.92      & 200.22$\pm$20.09              & 186.65$\pm$15.80      & 129.30$\pm$0.29 \\
5                   & 61.00                & 15.20$\pm$0.61     & 458.24$\pm$1.86       & 442.70$\pm$20.31               & 409.75$\pm$14.49      & 315.90$\pm$16.31 \\
6                   & 140.00               & 59.73             & 1017.30               & 854.15                        & 678.68                & 475.50 \\
\bottomrule
\end{tabular}}
\caption{Trend average and standard deviation values for different parameters of the Tower of Hanoi problem, namely the network depth \textit{d}, the total task planning time \textit{TP}, the total motion planning time \textit{MP} for the right arm, the number of motion planning attempts for the right arm, the total motion planning time \textit{MP} for the left arm, and the number of motion planning attempts for the left arm, as the number of objects varies from $3$ to $6$.}
\label{tab:hanoi_av}
\end{table}
\begin{figure}[H]
    \centering
    \includegraphics[scale=0.27]{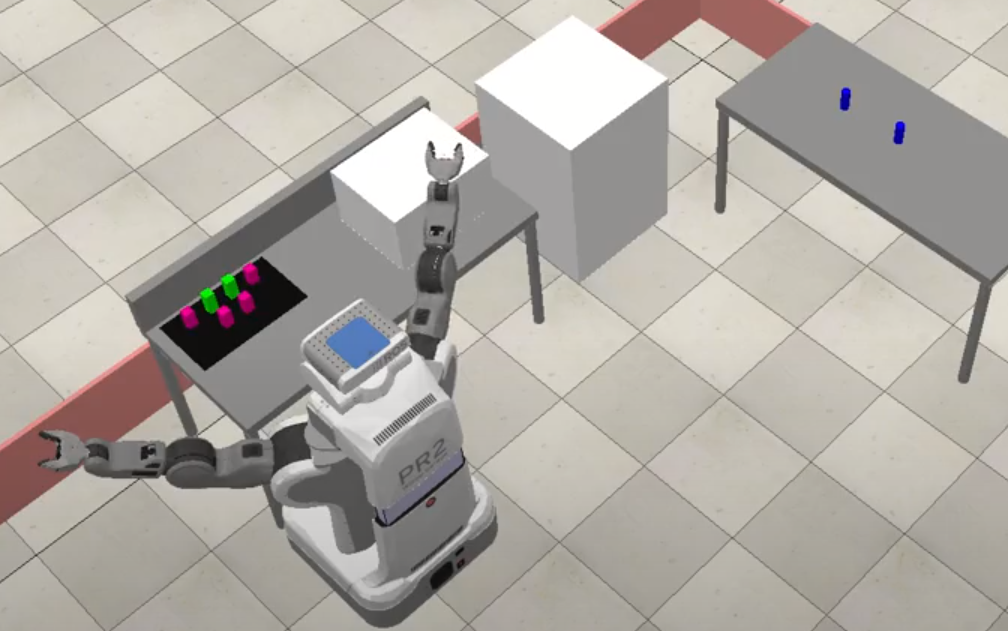}
    \caption{A PR2 robot at work in a kitchen environment. Cabbages are represented by green blocks, radishes by pink blocks, and glasses by blue blocks.}
    \label{fig:kit}
\end{figure}

\begin{figure}[H]
\centering
\subfloat[]{\includegraphics[width=3.9cm,height=2.9cm]{figures/kit_1.png}\label{fig:kit1}}\hfill
\subfloat[]{\includegraphics[width=3.9cm,height=2.9cm]{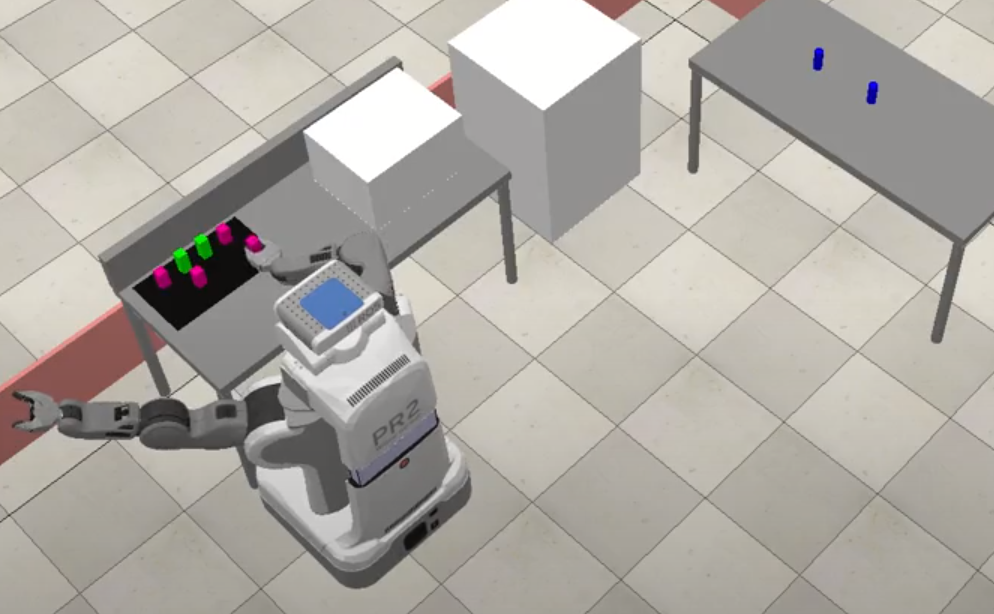}\label{fig:kit2}}\hfill
\subfloat[]{\includegraphics[width=3.9cm,height=2.9cm]{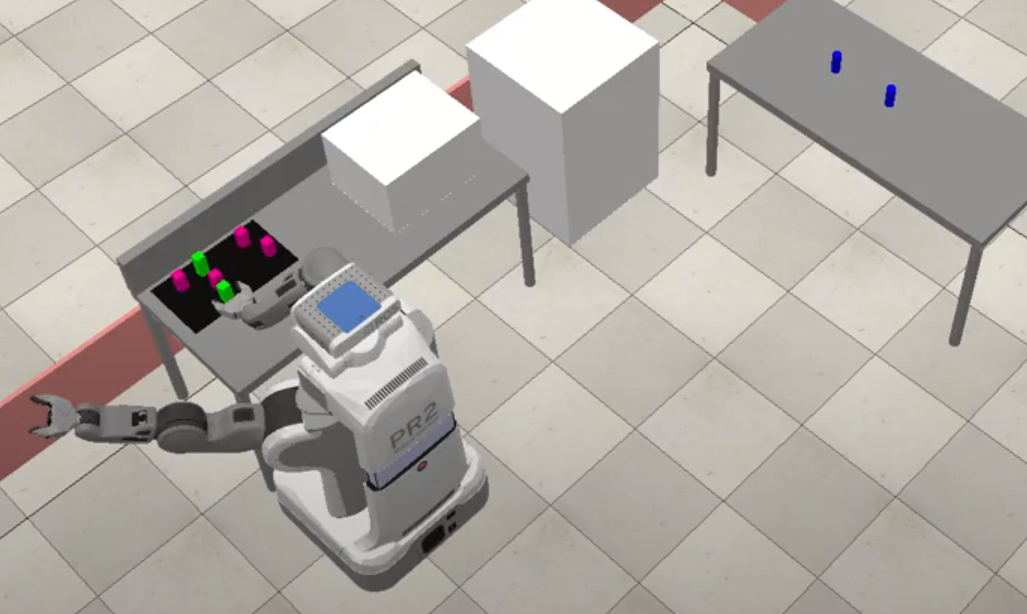}\label{fig:kit3}}\\
\subfloat[]{\includegraphics[width=3.9cm,height=2.9cm]{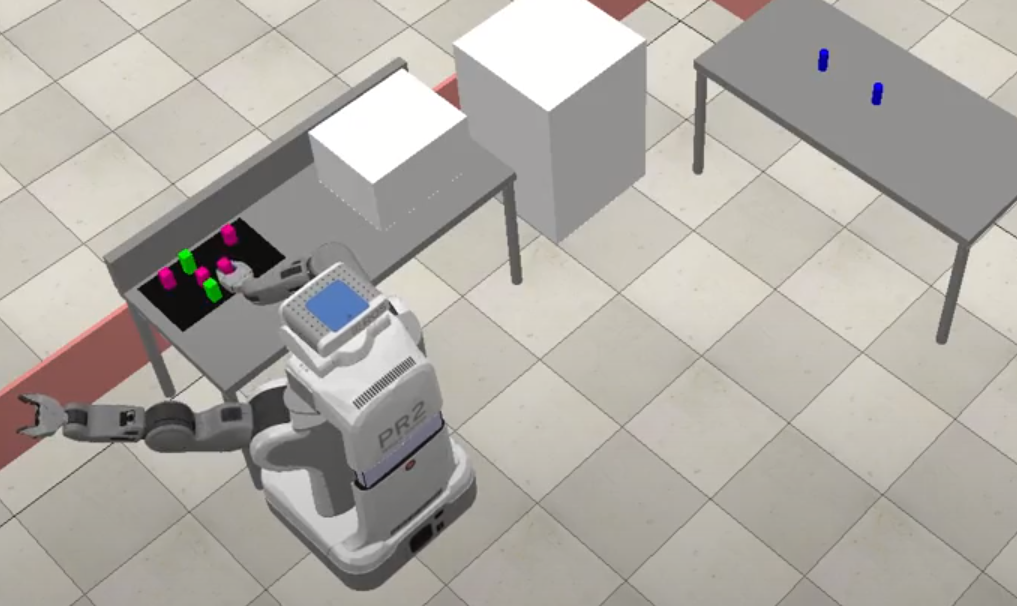}\label{fig:kit4}}\hfill
\subfloat[]{\includegraphics[width=3.9cm,height=2.9cm]{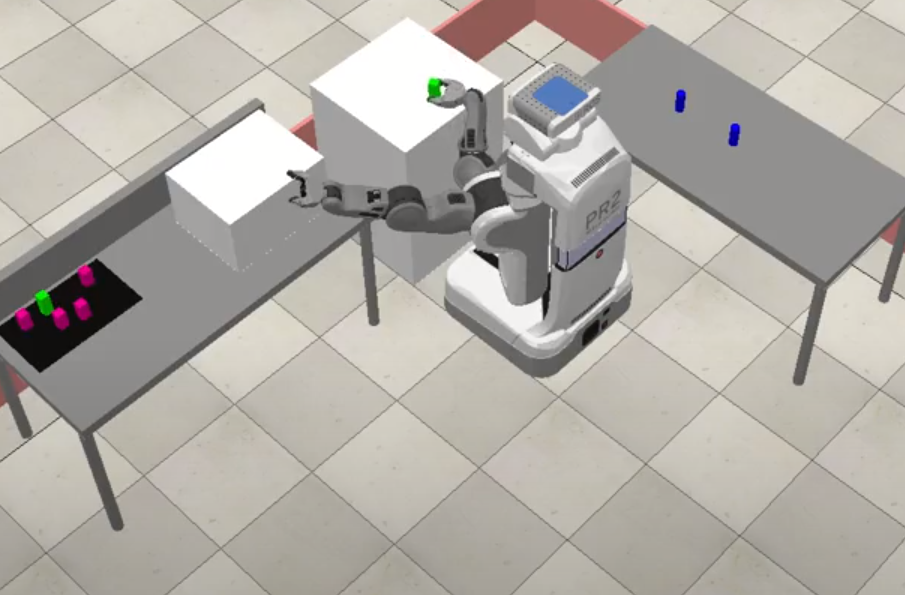}\label{fig:kit5}}\hfill
\subfloat[]{\includegraphics[width=3.9cm,height=2.9cm]{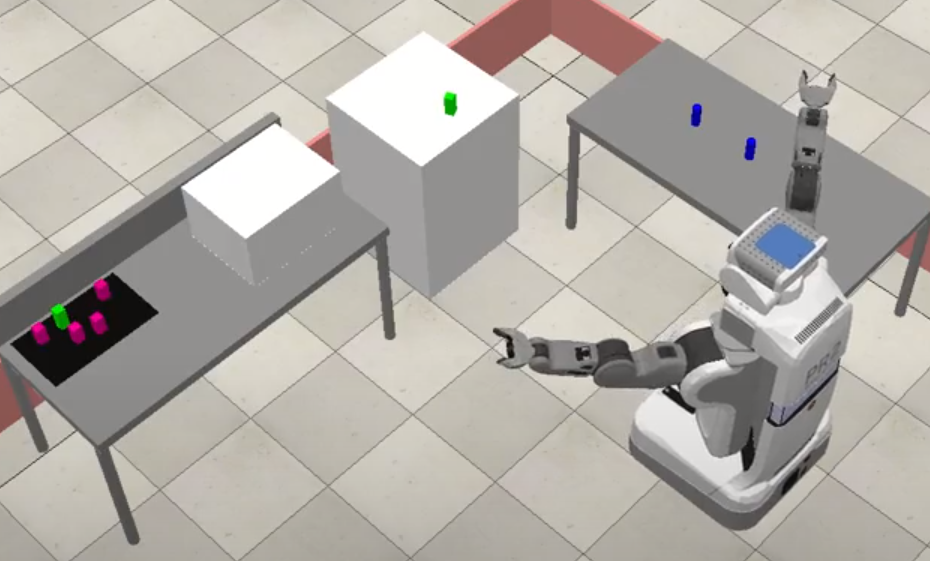}\label{fig:kit6}}\\
\subfloat[]{\includegraphics[width=3.9cm,height=2.9cm]{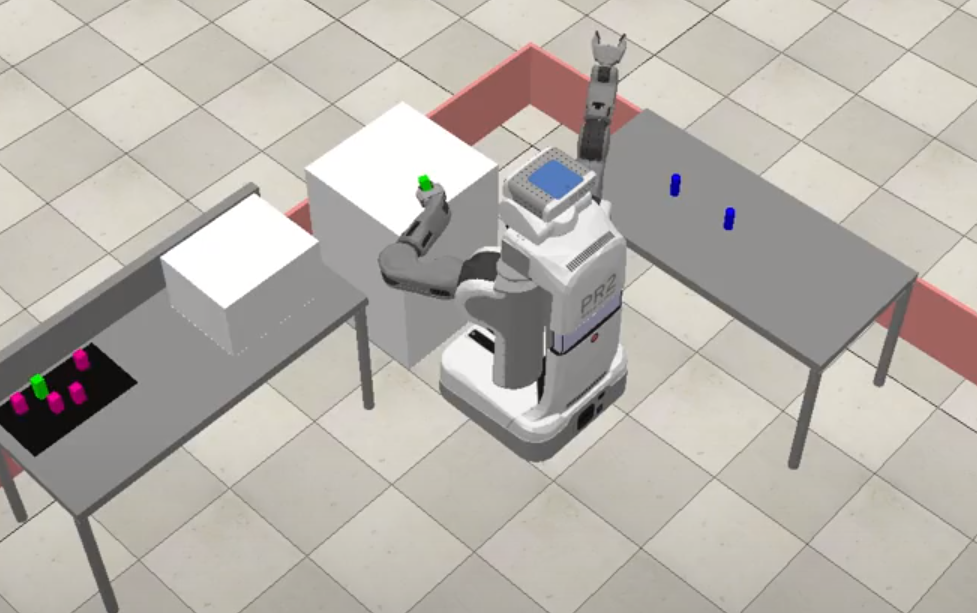}\label{fig:kit7}}\hfill
\subfloat[]{\includegraphics[width=3.9cm,height=2.9cm]{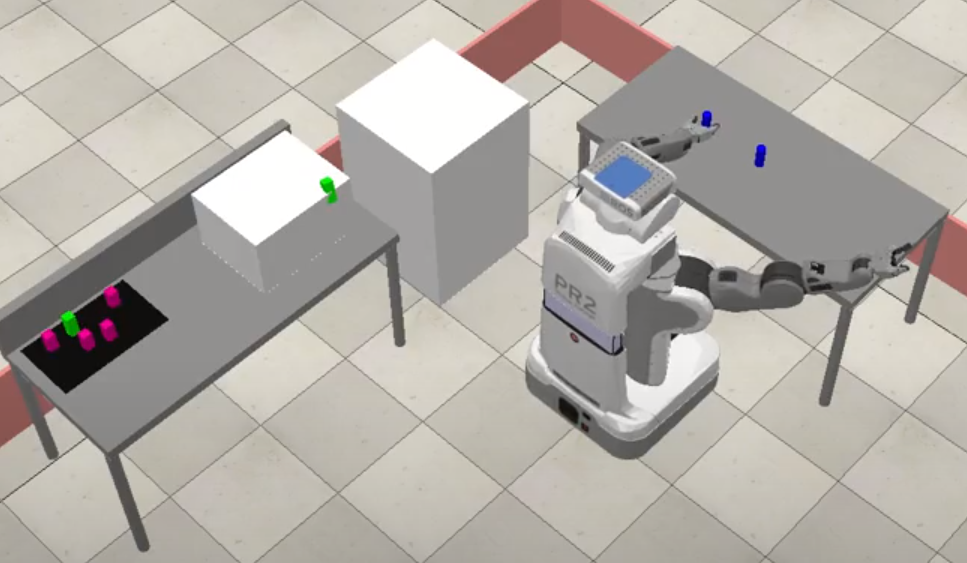}\label{fig:kit8}}\hfill
\subfloat[]{\includegraphics[width=3.9cm,height=2.9cm]{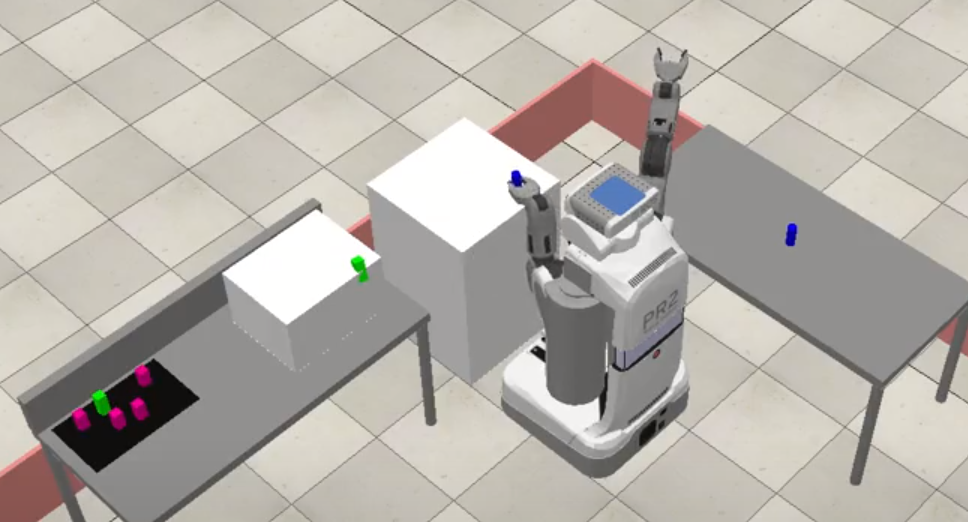}\label{fig:kit9}}\\
\subfloat[]{\includegraphics[width=3.9cm,height=2.9cm]{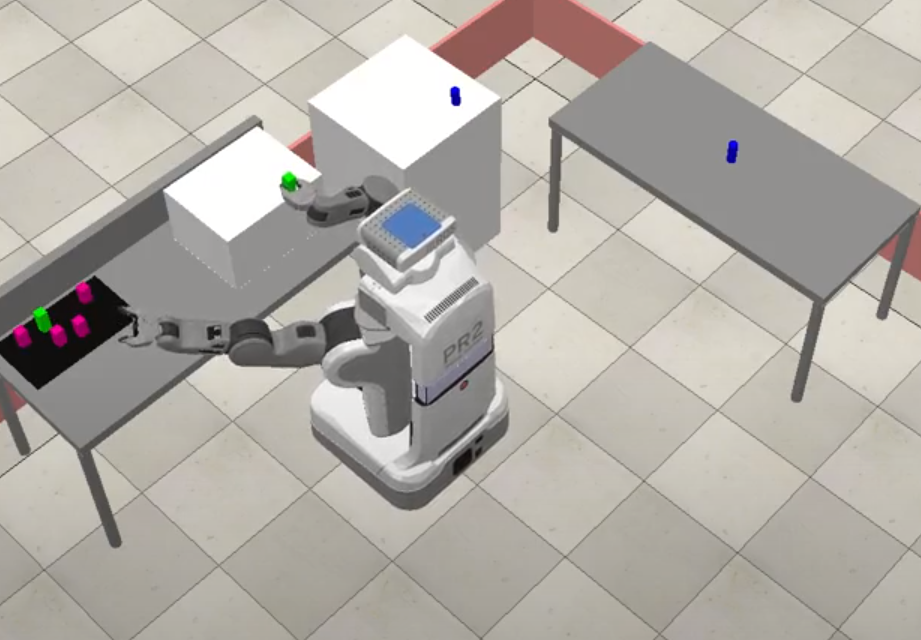}\label{fig:kit10}}\hfill
\subfloat[]{\includegraphics[width=3.9cm,height=2.9cm]{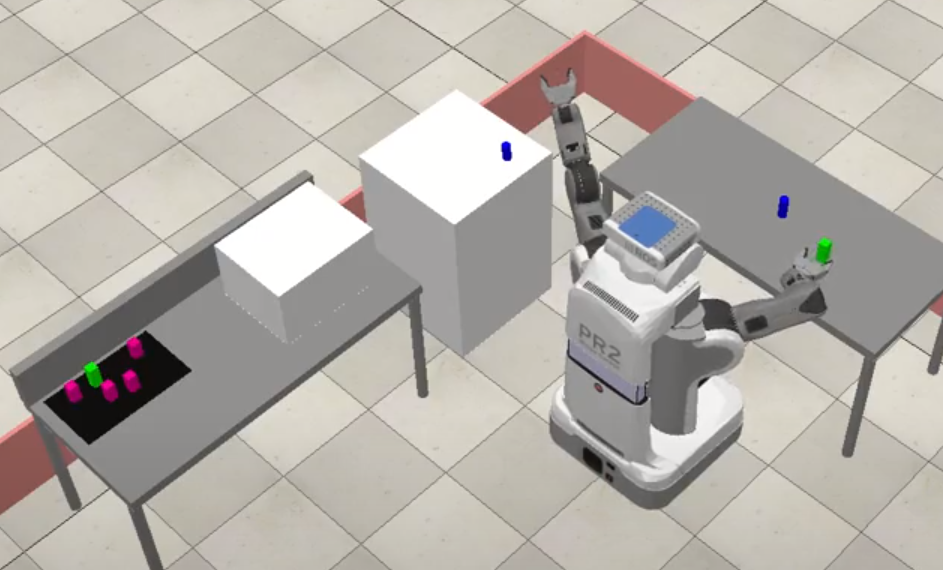}\label{fig:kit11}}\hfill
\subfloat[]{\includegraphics[width=3.9cm,height=2.9cm]{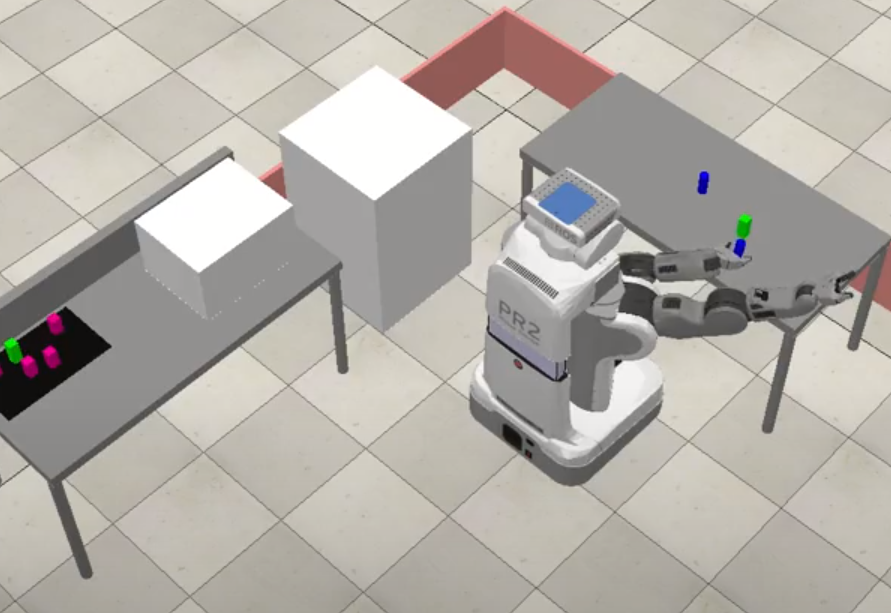}\label{fig:kit12}}\\
\caption{Part of the meal preparation in the kitchen environment. Refer to Section~\ref{subsub:kitchen} for a detailed explanation of the illustrations.}
\label{fig:kitchen}
\end{figure}

\subsubsection{Preparing a Meal in a Kitchen}
\label{subsub:kitchen}
The goal of this set of experiments is to evaluate the ability of TMP-IDAN to address a number of challenges in typical TAMP scenarios, namely:
the presence of unfeasible actions in a third domain, which is different from the ones related to the cluttered workspace and the Tower of Hanoi;
the need to deal with \textit{non-monotonic plans}, for example when an object must be picked-and-placed more than once;
the presence of \textit{non-geometric actions}, that is, actions which effects are only limited to the symbolic domain, without affecting the underlying configuration space.

In this case, we have set-up a scenario taking inspiration from a kitchen environment, whereas a PR2 robot (equipped with standard grippers and an RGB-D sensor) must \textit{prepare a meal} by cooking vegetables, cleaning glasses, and finally setting the table for two, as shown in Fig.~\ref{fig:kit}.
Blocks are used to mimic vegetables and glasses.
There are two types of vegetables, namely cabbages represented by green blocks, and  radishes represented by pink blocks.
Only cabbages must be cooked.
In the initial configuration, on a working table (on the left in Fig.~\ref{fig:kit}) radishes act as obstacles for picking up cabbages.
The radishes have to be moved aside to pick the cabbages, and finally they have to be placed back in their initial position. 
Glasses are represented by blue blocks.
There are two glasses, both located on a meal table (on the right in Fig.~\ref{fig:kit}).
Objects may be cleaned by inserting them in the dishwasher, that is, the big white cube in Fig.~\ref{fig:kit}. 
Cooking happens by placing vegetables in the microwave, that is, the small white cube on the working table in Fig.~\ref{fig:kit}. 
Fig.~\ref{fig:kitchen} shows various moments of the overall simulation. 
PR2 moves close to the table with the vegetables as shown in Fig.~\ref{fig:kit1}.
Then, it picks up an obstructing radish and places it aside on the table (see Fig.~\ref{fig:kit2}). 
PR2 then picks up one cabbage and places it in a predefined position, as shown in Fig.~\ref{fig:kit3}.
To keep the table tidy, the radish is placed back to its initial position (see Fig.~\ref{fig:kit4}).
Before the cabbage can be cooked, it must be washed. 
For this reason, it is placed in the dishwasher\footnote{We make two comments. The first is that the object is placed \textit{on}, and not \textit{in}, the dishwasher, which is due to the need to over-simplify the motions associated with this task. The second is that we do not enforce washing cabbages in the dishwasher as a best practice!}, as shown in Fig.~\ref{fig:kit5}.
Once the cabbage is placed in the dishwasher, PR2 waits for a while (to animate the \textit{wait} action, PR2 moves around and comes back as shown in Fig.~\ref{fig:kit6}). 
Later, Fig.~\ref{fig:kit7} shows that the cabbage is placed in (\textit{on}) the microwave for cooking.
While the cabbage is being cooked, a glass is picked up from the table on the right, and inserted into the dishwasher for cleaning as shown in Fig.~\ref{fig:kit8}-\ref{fig:kit9}.
Once the cabbage is cooked, it is picked from the microwave (see Fig.~\ref{fig:kit10}), and placed on the dinner table as shown in Fig.~\ref{fig:kit11}.
The glass is then picked from the dishwasher, and placed near the cabbage as shown in Fig.\ref{fig:kit12}.
This procedure is repeated with the remaining cabbage and glass until the table is set for two.
The simulation has been run $10$ times for statistical purposes.

\begin{figure}[]
    \centering
    \includegraphics[scale=0.6]{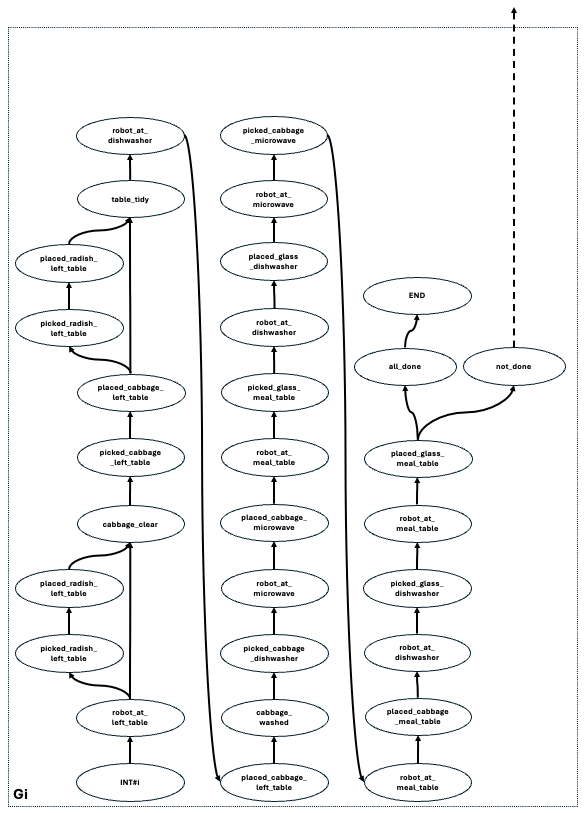}
    \caption{Part of an unfolded AND/OR graph used to model a single step of operations in the kitchen environment.}
    \label{fig:kitchen_graph}
\end{figure}
Fig.~\ref{fig:kitchen_graph} shows an unfolded AND/OR graph for the kitchen environment.
This AND/OR graph model replicates the behavior shown in Fig.~\ref{fig:kit}. 
However, it is important to note that the version depicted in the Figure is an \textit{unfolded}, simplified version of a more complex AND/OR graph.
Such an AND/OR graph would take into account, for example, the fact that the two sequences of
(i) picking up, washing, and putting down one cabbage on the meal table, and of 
(ii) picking up, washing, and putting down one glass on the meal table 
may be executed in any order, whereas in the unfolded version the first sequence is executed always before the second.   
Furthermore, as foreseen in \cite{darvish2020hierarchical}, a hierarchical AND/OR graph may be used to encode sub-sequences, such as for example the one related to picking up and putting down one radish, which in the unfolded version appears twice.
Also in this unfolded version, all PR2 actions are encoded within the hyper-arcs, which are all in \textit{OR}.
Some of the actions, for example, \textit{cook} and \textit{was} are non-geometric in nature, since they correspond to processes acting upon objects which configuration space does not change while the action unfolds. 
Since in this scenario PR2 must also move among different areas, we assume to have three \textit{agents}, namely PR2's right and left arms, as well as its mobile base.
Agents responsible for manipulation are grounded to manipulation actions on the basis of geometrical considerations, that is, which is the best arm to pick up a given vegetable or glass is computed at run-time before the action is executed, whereas all navigation actions are assigned to the PR2's mobile base.
 \begin{table}[t!]
\centering
\scalebox{0.7}{
\begin{tabular}{l c c c}
\hline
\textit{Computation time}   & \textit{Average} [s]  & \textit{Standard Deviation} [s] \\ 
\hline
\hline
AND/OR Graph                & 7.78                  & 0.47 \\
Graph Net Search            & 0.47                  & 0.01 \\
Motion Planner (right arm)  & 317.89                & 9.52 \\
Motion Planner (left arm)   & 121.76                & 20.95 \\
Motion Planner (base)       & 56.86                 & 0.02 \\
\hline
\end{tabular}}
\caption{Computation times for selected modules of TMP-IDAN, and their corresponding standard deviations, for the kitchen environment.}
\label{tab:kitchen}
\end{table}
Computational times of a selection of modules of \textsc{TMP-IDAN} are given in Table~\ref{tab:kitchen}.
Again, it is manifest how task planning is computationally negligible if compared with motion planning. 
This is due, like in the two previous scenarios, to the fact that the motion planners must explore a planning space more complex than the corresponding task space, and that multiple motion strategies must be explored to obtain a feasible motion.
This can be appreciated by comparing the AND/OR graph depth \textit{d}, which on average is $d=13.5$ with a standard deviation of $4.95$, and the number of planning attempts carried out to obtain feasible motions for the right arm, the left arm, and the mobile base.
The \textit{Motion Planner} module generates an average of $284.06$ attempts for the right arm, with a standard deviation of $29.07$, $84.11$ attempts for the left arm, with a standard deviation of $12.76$, and $57.83$ attempts for the mobile base, with a standard deviation of $1.62$.
In all the cases, the average number of attempts is significantly higher than the average depth of the AND/OR graph.
It is also possible to appreciate, as one would expect, that the planning of navigation motions is comparatively lightweight with respect to the case of manipulation motions.
\subsection{Multi-robot \textsc{TMP-IDAN}}
The goal of this class of experiments is to characterize the computational performance of multi-robot \textsc{TMP-IDAN}, and in particular with respect to parameters related to a \textit{large task space}, the presence of \textit{unfeasible actions}, and the employed task allocation strategy.

We have set-up a multi-robot cluttered table-top scenario, similar to one for the cluttered workspace domain, as shown in Fig.~\ref{fig:toy2}. 
In this case, all objects are in the form of cylinders which could be grasped only by their side.
The goal is to reach a specific target object, while re-arranging all objects which presence hinders the possibility of reaching and grasping the target object.
Experiments are performed using two Franka Emika Panda, referred to as $R1$ and $R2$, arranged in a squared configuration, and provided with standard grippers as well as with RGB-D cameras for detecting objects in their workspace.
Re-arranged objects are moved outside the workspace in a \textit{safe} space. 
The number of cylindrical objects varies from $6$ to $64$
For any given number of objects, the simulation is conducted $3$ times, and in each experiment the target objects are chosen randomly.
The employed AND/OR graph is not substantially different from the one reported in Fig.~\ref{fig:combined} on the right hand side, and therefore it is not reported here. 

\begin{table}[t!]
\centering
\scalebox{0.7}{
\begin{tabular}{c c c c c c} 
\hline
\textit{Objects}     & Avg. \textit{d}   & \textit{TP} [s]   & \textit{MP} [s]   & \textit{MP attempts}  & Objects to re-arrange \\
\hline
\hline
6           & 2.33              & 1.55             & 16.33            & 36.66                 & 1.66 \\
8           & 2.66              & 1.42             & 13.72            & 32.66                 & 1.66 \\
9           & 2.33              & 1.35             & 16.90            & 30.66                 & 1.00 \\
12          & 4.00              & 2.40             & 20.33            & 42.66                 & 2.66 \\
16          & 4.66              & 2.75             & 20.17            & 47.66                 & 3.33 \\
20          & 3.66              & 2.12             & 17.61            & 42.33                 & 2.66 \\
30          & 5.75              & 3.46             & 25.85            & 56.75                 & 4.25 \\
49          & 17.00             & 8.64             & 81.36            & 147.00                & 6.33 \\
64          & 22.6              & 9.05             & 132.17           & 203.00                & 12.2 \\
\hline
\end{tabular}}
\caption{Trend of the average network depth \textit{d}, the average total task planning time \textit{TP}, the average total motion planning time \textit{MP}, the number of average motion planning attempts and the average number of objects to be re-arranged, as the number of objects varies from $6$ to $64$ in the multi-robot cluttered workspace for $R1$.}
\label{tab:table1}
\end{table}

\begin{table}[t!]
\centering
\scalebox{0.7}{
\begin{tabular}{c c c c c c} 
\hline
\textit{Objects}     & Avg. \textit{d}   & \textit{TP} [s]   & \textit{MP} [s]   & \textit{MP attempts}  & Objects to re-arrange \\
\hline
\hline
6           & 2.33              & 1.12              & 13.20             & 31.33                 & 1.00 \\
8           & 2.00              & 1.14              & 18.15             & 28.66                 & 1.00 \\
9           & 3.66              & 2.41              & 18.52             & 40.00                 & 2.33 \\
12          & 3.33              & 2.05              & 17.82             & 37.66                 & 2.33 \\
16          & 4.33              & 2.63              & 22.83             & 45.66                 & 3.33 \\
20          & 6.33              & 4.02              & 25.18             & 61.33                 & 5.33 \\
30          & 9.25              & 4.74              & 45.22             & 83.75                 & 5.75 \\
49          & 12.33             & 6.88              & 49.59             & 106.66                & 10.66 \\
64          & 12.60             & 5.27              & 66.99             & 109.87                & 7.00 \\
\hline
\end{tabular}}
\caption{Trend of the average network depth \textit{d}, the average total task planning time \textit{TP}, the average total motion planning time \textit{MP}, the number of average motion planning attempts and the average number of objects to be re-arranged, as the number of objects varies from $6$ to $64$ in the multi-robot cluttered workspace for $R2$.}
\label{tab:table2}
\end{table}

Table~\ref{tab:table1} and Table~\ref{tab:table2} report the average network depth $d$, the average total task planning time \textit{TP}, the average total motion planning time \textit{MP}, the average number of motion planning attempts, and the average number of objects to be re-arranged for robots $R1$ and $R2$, respectively. 
It can be observed that, as the AND/OR graph network depth $d$ increases, task planning times are \textit{almost} linear with respect to $d$, the reason being the fact that, for an AND/OR graph network with each graph consisting of $n$ nodes, the time complexity is $O(nd)$. 
In contrast, PDDL-based planners are characterized by a search complexity of $O(n \log n)$, where $n \approx 2^{13}$ for this table-top scenario. 

\begin{table}[t!]
\centering
\scalebox{0.7}{
\begin{tabular}{lcc} 
\hline
\textit{Computation time}   & \textit{Average} [s]      & \textit{Standard Deviation} [s] \\ 
\hline
\hline
AND/OR Graph                & 0.03213                   & 0.02174 \\ 
Graph Net Search            & 0.8992                    & 0.1862 \\
Motion Planner              & 0.8010                    & 0.2170 \\
Motion Executor             & 4.7490                    & 2.0240 \\
\hline
\end{tabular}}
\caption{Computation times for selected modules of \textsc{TMP-IDAN}, and their corresponding standard deviations, for the multi-robot case in a cluttered workspace with $6$ objects.}
\label{tab:compana}
\end{table}
Table~\ref{tab:compana} shows the average combined planning and execution times for robots $R1$ and $R2$, focusing on a selection of the modules in the architecture.
In this case, the average running time of the \textit{AND/OR Graph} and \textit{Graph Net Search} modules is comparable to the one of the \textit{Motion Planner}, despite the need to run many motion planning attempts (an average of $18.32$ with a standard deviation of $6.45$)
Both the discrete and continuous planning processes are fast when compared with the running times of \textit{Motion Executor}.

\begin{figure}[H]
\centering
\subfloat[]{\includegraphics[width=7cm]{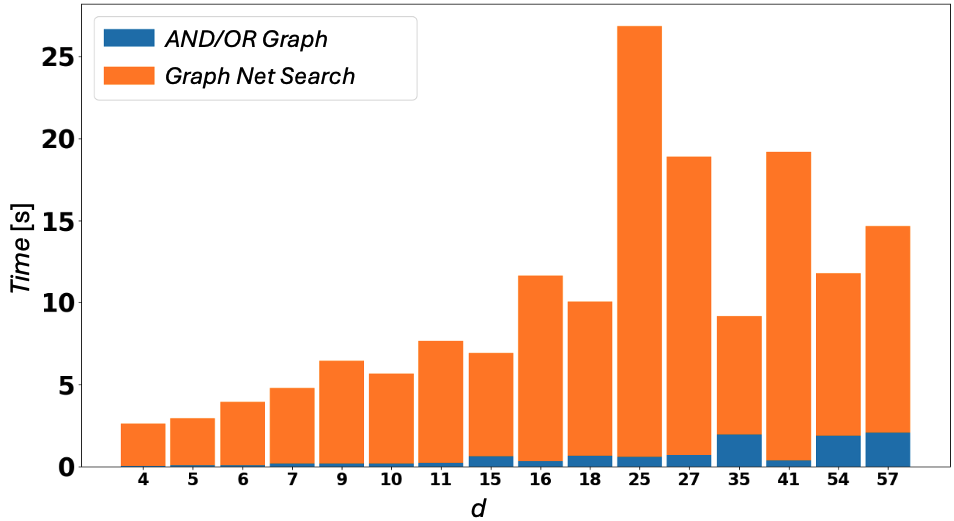}\label{fig:compa1}}\\
\subfloat[]{\includegraphics[width=7cm]{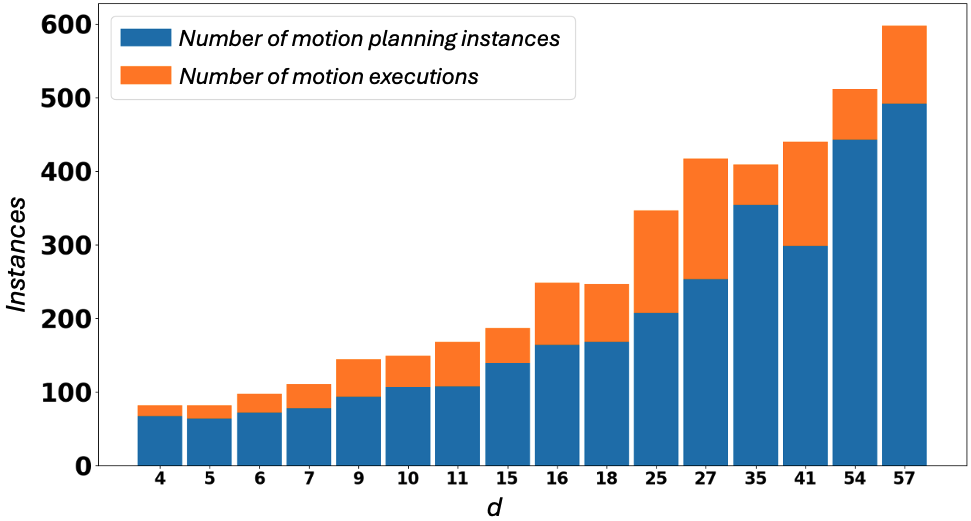}\label{fig:compb1}}\\
\subfloat[]{\includegraphics[width=7cm]{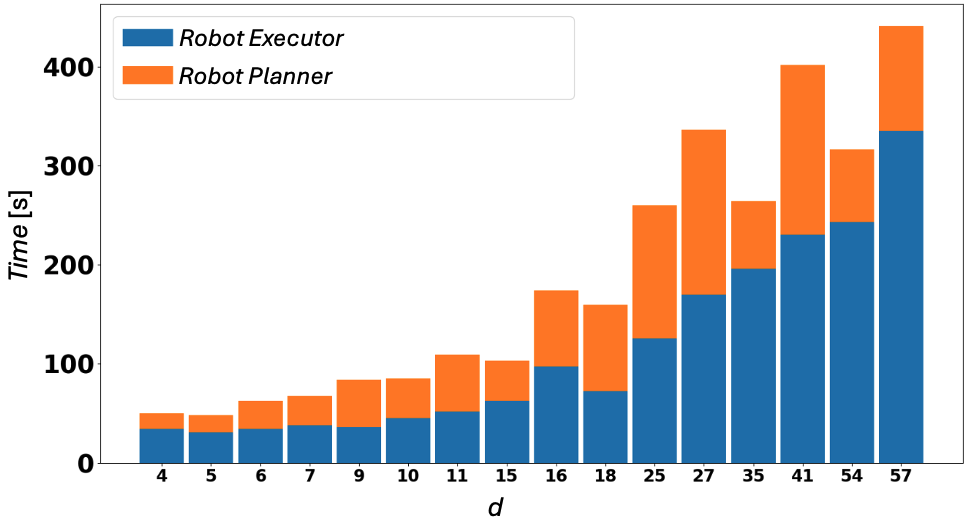}\label{fig:compc1}}
\caption{Various trends associated with an increasing number of depth \textit{d} for the multi-robot case in a cluttered workspace: 
(a) the computational time of \textit{Graph Net Search} dominates the one of \textit{AND/OR Graph};
(b) the number of motion planning instances is higher than actual executions, since \textit{Motion Planner} tries various motion strategies before them being executed;
(c) the increase in \textit{d} has a small effect on the overall planning and execution time, as execution takes significantly more time over planning.}
\label{fig:comp1}
\end{figure}
Fig.~\ref{fig:comp1} shows different plot for an increasing network depth $d$. 
In particular, Fig.~\ref{fig:compa1} shows the total task planning time with $d$. 
It can be observed that the linearity hypothesis in planning time is not strictly followed since, for example, the time for $d=25$ is greater than the time for $d=41$. 
In a number of simulations, due to motion planning failures, a new graph is expanded before reaching its node, which implies that more nodes are traversed for $d=25$ compared to $d=41$, therefore explaining the variations. 
Fig.~\ref{fig:compb1} plots the number of motion planning attempts and the total motion executions with an increasing $d$. 
An increase in $d$ in most cases correspond to a higher number of obstructing objects. 
Therefore, as the depth $d$ increases, there is a growth in the number of motion planning attempts.
However, motion planning failures can also increase the depth $d$ since a new graph needs to be expanded. 
This explains the slight deviation in the trend. 
Fig.~\ref{fig:compc1} reports the total motion planning and execution times with an increasing $d$, and we can notice how the plot follows from the discussion above. 

\section{Discussion}
\label{sec:discussion}
As we have discussed in the previous Section, \textsc{TMP-IDAN} is capable of solving various challenging TAMP problems, both in the single-robot and the multi-robot case.
However, the current version of TMP-IDAN suffers from a few limitations. 
In this Section, we briefly highlight them, with the aim of pointing towards interesting directions for future research.

In the real-world experimental scenario, we did not employ a precise object detection system. The positions of the objects were known beforehand, and a perception system was used to distinguish between normal objects and the goal object (marked with red tape in Fig.~\ref{fig:combined}). 

\textit{Possible non-monotonicity of the task-motion planning process}.
We have seen that, for example in the \textit{Cluttered workspace} scenario, it may happen that objects hindering a pick-and-place task are removed from the workspace.
In a sense, as soon as they are removed from the workspace, they cannot be considered anymore at the task planning level.
However, situations may occur when removing objects from the workspace may not be possible, and the motion planner should reason upon where to store re-arranged objects.
In this case, the object is still part of the domain, and can be re-considered in future tasks.
In our case, this happens in the \textit{Kitchen} scenario, where re-arranging vegetables implies placing them on the same working table for later use.
To temporarily store re-arranged objects, the motion planner could, in principle, randomly select free space for the placements. 
However, this choice would likely affect the utility (and therefore the overall plan quality) since a future robot trajectory could collide again with the re-arranged object, and this would trigger a new re-arrangement.
Obviously enough, this is a specific instance of the more general case on non-monotonicity, whereby the outcome of an action may impact other downstream actions in the whole plan \cite{capitanelli2018RAS}.
The case of non-monotone actions is a well-studied issue in symbolic planning, and it is epitomized in the famous Sussman anomaly.
However, in this case, it originates from the interplay between the task-level and the motion-level planning, and therefore it is not caused by issues in the pure symbolic representation.
Current work seeks to develop a non-myopic technique to predict possible future trajectories based on previous motions and the current environment state, so as to enable the selection of object placements minimizing non-monotonicity. 

\textit{A fully observable, deterministic, and static workspace}.
Currently, TMP-IDAN assumes a fully observable, deterministic environment with static obstacles. 
In order to relax this assumption, different research directions could be pursued.
Being able to handle partially observable scenarios could be achieved by enabling the motion planner to use planning under uncertainty techniques.
While this is certainly possible, and related approaches have been discussed in the literature, their adoption would be only partially useful.
In fact, it would be necessary also to adopt probabilistic decision making techniques, that is, to modify the AND/OR graph network layer \cite{darvish2018Mechatronics, darvish2020hierarchical, karami2020task, karami2020ROMAN}
While AND/OR graphs allow for a weak notion of uncertainty by modeling possible alternative pathways to reach a goal state, the structure of the graph is still fixed.
A promising direction would be to integrate AND/OR graph networks with full Bayesian state estimation techniques, embracing such approaches as Active Inference \cite{Fristonetal2017} or Decision Field Theory \cite{BusemeyerandTownsend1993}.
Similarly, the incorporation of dynamic obstacles, including humans moving within the robot's workspace, could be very important to address non-structured scenarios related, for example, to human-robot collaboration.
This would require to update the \textit{Knowledge Base} module frequently, as driven by continuous information from the \textit{Scene Perception} module, which should be extended to encompass the prediction of object and human (navigation and manipulation) trajectories.
Recent research work on digital twin approaches for human-robot collaboration \cite{Shaabanetal2023}, as well human motion modeling, classification and prediction \cite{Carfietal2018, CarfiandMastrogiovanni2023} seem to go in this promising direction.

Both of these topics are intriguing directions for future research, and current work is undergoing to address their challenges. 
\section{Conclusion}
\label{sec:conclusion}
In this paper, we present TMP-IDAN, a novel approach to task and motion planning, which is based on the hypothesis that it is possible to reduce a wide range of TAMP problems to a recursive formulation, which is expanded while the execution of the task and motion plan unfolds.
On the one hand, the task-level representation is based on an AND/OR graph, a formalism that reduces a task planning search space to a small, countable set of engineered alternatives, by constraining the trajectories of possible solutions in the symbolic state space.
Whenever a complex problem can be divided into a sequential execution of a (not known in advance) number of similar steps, it is possible to grow a network of AND/OR graphs thereby interleaving task planning and execution.
On the other hand, motion planning is implemented using a variant of RRT.
For each task-level action, different motion planning attempts are made via an in-the-loop simulation, which maximizes the likelihood of successful robot motions.

We discuss different aspects of TMP-IDAN.
First, we characterize its behavior on the basis of a number of aspects, such as unfeasible actions due to failures at the motion planning or execution level, large task spaces, the interplay between tasks and motions, non-monotone plans, and actions characterized solely by symbolic effects, that is, without consequences at the motion level.
Second, we introduce two variants of TMP-IDAN, namely single- and multi-robot.
Third, we characterize the overall computational complexity of TMP-IDAN as almost linear, and we compare it with the \textit{de facto} standard in planning based on the PDDL language.
Fourth, we validate those aspects in a number of well-known scenarios in TAMP, such as retrieving an object from a cluttered workspace, solving the Tower of Hanoi game, and preparing a meal in a kitchen environment. 

Among possible promising research direction, we identify the need to deal with non-monotone plans in the integrated TAMP search spaces, and the extension to partially observable and dynamic environments.

\section*{Acknowledgments}
This research is partially supported by the Italian government under the National Recovery and Resilience Plan (NRRP), Mission 4, Component 2, Investment 1.5, funded by the European Union NextGenerationEU programme, and awarded by the Italian Ministry of University and Research, project RAISE, grant agreement no. J33C22001220001.
 
\bibliographystyle{elsarticle-num}
\bibliography{References}

\begin{thebibliography}{10}
\expandafter\ifx\csname url\endcsname\relax
  \def\url#1{\texttt{#1}}\fi
\expandafter\ifx\csname urlprefix\endcsname\relax\def\urlprefix{URL }\fi
\expandafter\ifx\csname href\endcsname\relax
  \def\href#1#2{#2} \def\path#1{#1}\fi

\bibitem{lagriffoul2018RAL}
F.~Lagriffoul, N.~T. Dantam, C.~Garrett, A.~Akbari, S.~Srivastava, L.~E.
  Kavraki, Platform-independent benchmarks for task and motion planning,
  Robotics and Automation Letters.

\bibitem{garrett2021ARC}
C.~R. Garrett, R.~Chitnis, R.~Holladay, B.~Kim, T.~Silver, L.~P. Kaelbling,
  T.~Lozano-P{\'e}rez, Integrated task and motion planning, Annual review of
  control, robotics, and autonomous systems 4~(1) (2021) 265--293.

\bibitem{guo2023ACM}
H.~Guo, F.~Wu, Y.~Qin, R.~Li, K.~Li, K.~Li, Recent trends in task and motion
  planning for robotics: A survey, ACM Computing Surveys 55~(13s) (2023) 1--36.

\bibitem{mcdermott1998AIPS}
D.~McDermott, M.~Ghallab, A.~Howe, C.~Knoblock, A.~Ram, M.~Veloso, D.~Weld,
  D.~Wilkins, {PDDL- The Planning Domain Definition Language}, in: AIPS-98
  Planning Competition Committee, 1998.

\bibitem{he2015ICRA}
K.~He, M.~Lahijanian, L.~E. Kavraki, M.~Y. Vardi, Towards manipulation planning
  with temporal logic specifications, in: 2015 IEEE international conference on
  robotics and automation (ICRA), IEEE, 2015, pp. 346--352.

\bibitem{dantam2013TRO}
N.~Dantam, M.~Stilman, The motion grammar: Analysis of a linguistic method for
  robot control, IEEE Transactions on Robotics 29~(3) (2013) 704--718.

\bibitem{ghallab2016book}
M.~Ghallab, D.~Nau, P.~Traverso, Automated planning and acting, Cambridge
  University Press, 2016.

\bibitem{latombe1991robot}
J.-C. Latombe, Robot Motion Planning, Kluwer Academic Publishers, 1991.

\bibitem{bryce2007AIM}
D.~Bryce, S.~Kambhampati, A tutorial on planning graph based reachability
  heuristics, AI Magazine 28~(1) (2007) 47--47.

\bibitem{bertolucci2021TPLP}
R.~Bertolucci, A.~Capitanelli, C.~Dodaro, N.~Leone, M.~Maratea,
  F.~Mastrogiovanni, M.~Vallati, Manipulation of articulated objects using
  dual-arm robots via answer set programming, Theory and Practice of Logic
  Programming 21~(3) (2021) 372--401.

\bibitem{helmert2006PHD}
M.~Helmert, Solving planning tasks in theory and practice, Ph.D. thesis,
  Albert-Ludwigs-Universität Freiburg (2006).

\bibitem{motes2020RAL}
J.~Motes, R.~Sandstr{\"o}m, H.~Lee, S.~Thomas, N.~M. Amato, Multi-robot task
  and motion planning with subtask dependencies, IEEE Robotics and Automation
  Letters 5~(2) (2020) 3338--3345.

\bibitem{karami2021AIIA}
H.~Karami, A.~Thomas, F.~Mastrogiovanni, {Task Allocation for Multi-robot Task
  and Motion Planning: A Case for Object Picking in Cluttered Workspaces}, in:
  {AIxIA 2021 -- Advances in Artificial Intelligence}, Springer International
  Publishing, Cham, 2022, pp. 3--17.

\bibitem{kaelbling2013IJRR}
L.~P. Kaelbling, T.~Lozano-P{\'e}rez, Integrated task and motion planning in
  belief space, The International Journal of Robotics Research 32~(9-10) (2013)
  1194--1227.

\bibitem{srivastava2014ICRA}
S.~Srivastava, E.~Fang, L.~Riano, R.~Chitnis, S.~Russell, P.~Abbeel, Combined
  task and motion planning through an extensible planner-independent interface
  layer, in: Robotics and Automation (ICRA), IEEE International Conference on,
  IEEE, 2014, pp. 639--646.

\bibitem{dantam2016RSS}
N.~T. Dantam, Z.~K. Kingston, S.~Chaudhuri, L.~E. Kavraki, Incremental task and
  motion planning: A constraint-based approach, in: Robotics: Science and
  Systems, 2016.

\bibitem{lagriffoul2016IJRR}
F.~Lagriffoul, B.~Andres, Combining task and motion planning: A culprit
  detection problem, The International Journal of Robotics Research 35~(8)
  (2016) 890--927.

\bibitem{garrett2018IJRR}
C.~R. Garrett, T.~Lozano-Perez, L.~P. Kaelbling, {FFRob: Leveraging symbolic
  planning for efficient task and motion planning}, The International Journal
  of Robotics Research 37~(1) (2018) 104--136.

\bibitem{thomas2021RAS}
A.~Thomas, F.~Mastrogiovanni, M.~Baglietto,
  \href{https://www.sciencedirect.com/science/article/pii/S0921889021000713}{{MPTP:
  Motion-planning-aware task planning for navigation in belief space}},
  Robotics and Autonomous Systems 141 (2021) 103786.
\newblock \href {http://dx.doi.org/https://doi.org/10.1016/j.robot.2021.103786}
  {\path{doi:https://doi.org/10.1016/j.robot.2021.103786}}.
\newline\urlprefix\url{https://www.sciencedirect.com/science/article/pii/S0921889021000713}

\bibitem{garrett2019arxiv}
C.~R. Garrett, C.~Paxton, T.~Lozano-P{\'e}rez, L.~P. Kaelbling, D.~Fox, Online
  replanning in belief space for partially observable task and motion problems,
  arXiv preprint arXiv:1911.04577.

\bibitem{cambon2009IJRR}
S.~Cambon, R.~Alami, F.~Gravot, A hybrid approach to intricate motion,
  manipulation and task planning, The International Journal of Robotics
  Research 28~(1) (2009) 104--126.

\bibitem{kaelbling2012aTR}
L.~P. Kaelbling, T.~Lozano-P{\'e}rez, Integrated robot task and motion planning
  in the now, Tech. Rep. 2012-018, Computer Science and Artificial Intelligence
  Laboratory, Massachusetts Institute of Technology (2012).

\bibitem{pandey2012BIOROB}
A.~K. Pandey, J.-P. Saut, D.~Sidobre, R.~Alami, Towards planning human-robot
  interactive manipulation tasks: Task dependent and human oriented autonomous
  selection of grasp and placement, in: 2012 4th IEEE RAS \& EMBS International
  Conference on Biomedical Robotics and Biomechatronics (BioRob), IEEE, 2012,
  pp. 1371--1376.

\bibitem{de2013RSSws}
L.~de~Silva, A.~K. Pandey, M.~Gharbi, R.~Alami, Towards combining htn planning
  and geometric task planning, in: RSS Workshop on Combined Robot Motion
  Planning and AI Planning for Practical Applications, 2013.

\bibitem{dornhege2009SSRR}
C.~Dornhege, M.~Gissler, M.~Teschner, B.~Nebel, Integrating symbolic and
  geometric planning for mobile manipulation, in: Safety, Security \& Rescue
  Robotics (SSRR), IEEE International Workshop on, IEEE, 2009, pp. 1--6.

\bibitem{dornhege2009ICAPS}
C.~Dornhege, P.~Eyerich, T.~Keller, S.~Tr{\"u}g, M.~Brenner, B.~Nebel,
  {Semantic Attachments for Domain-Independent Planning Systems}, in:
  International Conference on Automated Planning and Scheduling (ICAPS),
  Thessaloniki, Greece, 2009, pp. 114--121.

\bibitem{kaelbling2011ICRA}
L.~P. Kaelbling, T.~Lozano-P{\'e}rez, Hierarchical task and motion planning in
  the now, in: Robotics and Automation (ICRA), IEEE International Conference
  on, IEEE, 2011, pp. 1470--1477.

\bibitem{toussaint2015IJCAI}
M.~Toussaint, Logic-geometric programming: An optimization-based approach to
  combined task and motion planning, in: Twenty-Fourth International Joint
  Conference on Artificial Intelligence, 2015.

\bibitem{erdem2011ICRA}
E.~Erdem, K.~Haspalamutgil, C.~Palaz, V.~Patoglu, T.~Uras, Combining high-level
  causal reasoning with low-level geometric reasoning and motion planning for
  robotic manipulation, in: 2011 IEEE International Conference on Robotics and
  Automation, IEEE, 2011, pp. 4575--4581.

\bibitem{de2011CACM}
L.~De~Moura, N.~Bj{\o}rner, Satisfiability modulo theories: introduction and
  applications, Communications of the ACM 54~(9) (2011) 69--77.

\bibitem{garrett2020ICAPS}
C.~R. Garrett, T.~Lozano-P{\'e}rez, L.~P. Kaelbling, Pddlstream: Integrating
  symbolic planners and blackbox samplers via optimistic adaptive planning, in:
  Proceedings of the International Conference on Automated Planning and
  Scheduling, Vol.~30, 2020, pp. 440--448.

\bibitem{paulius2023RAL}
D.~Paulius, A.~Agostini, D.~Lee, Long-horizon planning and execution with
  functional object-oriented networks, IEEE Robotics and Automation Letters
  8~(8) (2023) 4513--4520.

\bibitem{rodriguez2016IROS}
C.~Rodr{\'\i}guez, R.~Su{\'a}rez, Combining motion planning and task assignment
  for a dual-arm system, in: 2016 IEEE/RSJ International Conference on
  Intelligent Robots and Systems (IROS), IEEE, 2016, pp. 4238--4243.

\bibitem{umay2019ROSE}
I.~Umay, B.~Fidan, W.~Melek, An integrated task and motion planning technique
  for multi-robot-systems, in: 2019 IEEE International Symposium on Robotic and
  Sensors Environments (ROSE), IEEE, 2019, pp. 1--7.

\bibitem{basile2012RCIM}
F.~Basile, F.~Caccavale, P.~Chiacchio, J.~Coppola, C.~Curatella, Task-oriented
  motion planning for multi-arm robotic systems, Robotics and
  Computer-Integrated Manufacturing 28~(5) (2012) 569--582.

\bibitem{henkel2019IROS}
C.~Henkel, J.~Abbenseth, M.~Toussaintl, An optimal algorithm to solve the
  combined task allocation and path finding problem, in: 2019 IEEE/RSJ
  International Conference on Intelligent Robots and Systems (IROS), IEEE,
  2019, pp. 4140--4146.

\bibitem{motes2019RAL}
J.~Motes, R.~Sandstr{\"o}m, W.~Adams, T.~Ogunyale, S.~Thomas, N.~M. Amato,
  Interaction templates for multi-robot systems, IEEE Robotics and Automation
  Letters 4~(3) (2019) 2926--2933.

\bibitem{thomas2020STAIRS}
A.~Thomas, F.~Mastrogiovanni, M.~Baglietto, {Towards Multi-Robot Task-Motion
  Planning for Navigation in Belief Space}, in: European Starting AI
  Researchers’ Symposium, CEUR, 2020.

\bibitem{sharon2015AI}
G.~Sharon, R.~Stern, A.~Felner, N.~R. Sturtevant, Conflict-based search for
  optimal multi-agent pathfinding, Artificial Intelligence 219 (2015) 40--66.

\bibitem{karami2020ROMAN}
H.~Karami, K.~Darvish, F.~Mastrogiovanni, {A Task Allocation Approach for
  Human-Robot Collaboration in Product Defects Inspection Scenarios}, in: 29th
  IEEE International Conference on Robot and Human Interactive Communication,
  2020.

\bibitem{chang1971AI}
C.-L. Chang, J.~R. Slagle, {An admissible and optimal algorithm for searching
  AND/OR graphs}, Artificial Intelligence 2~(2) (1971) 117--128.

\bibitem{sucan2013moveit}
I.~A. Sucan, S.~Chitta, Moveit! (2013).

\bibitem{kuffner2000ICRA}
J.~J. Kuffner, S.~M. LaValle, Rrt-connect: An efficient approach to
  single-query path planning, in: Robotics and Automation, 2000. Proceedings.
  ICRA'00. IEEE International Conference on, Vol.~2, IEEE, 2000, pp. 995--1001.

\bibitem{sanderson1988TAES}
A.~Sanderson, M.~Peshkin, L.~H. de~Mello, Task planning for robotic
  manipulation in space applications, IEEE transactions on aerospace and
  electronic systems 24~(5) (1988) 619--629.

\bibitem{karaman2011IJRR}
S.~Karaman, E.~Frazzoli, Sampling-based algorithms for optimal motion planning,
  The international journal of robotics research 30~(7) (2011) 846--894.

\bibitem{hauser2014IJRR}
K.~Hauser, The minimum constraint removal problem with three robotics
  applications, The International Journal of Robotics Research 33~(1) (2014)
  5--17.

\bibitem{thomas2023IAS}
A.~Thomas, F.~Mastrogiovanni, M.~Baglietto, {Revisiting the Minimum Constraint
  Removal Problem in Mobile Robotics}, in: Intelligent Autonomous Systems 18,
  Springer Nature Switzerland, Cham, 2024, pp. 31--41.

\bibitem{hauser2013RSS}
K.~Hauser, Minimum constraint displacement motion planning, in: Proceedings of
  Robotics: Science and Systems IX, Berlin, Germany, 2013.
\newblock \href {http://dx.doi.org/10.15607/RSS.2013.IX.017}
  {\path{doi:10.15607/RSS.2013.IX.017}}.

\bibitem{thomas2022IAS}
A.~Thomas, F.~Mastrogiovanni, {Minimum Displacement Motion Planning for Movable
  Obstacles}, in: Intelligent Autonomous Systems 17, Springer Nature
  Switzerland, Cham, 2023, pp. 155--166.

\bibitem{thomas2023ICRA}
A.~Thomas, G.~Ferro, F.~Mastrogiovanni, M.~Robba, Computational tradeoff in
  minimum obstacle displacement planning for robot navigation, in: 2023 IEEE
  International Conference on Robotics and Automation (ICRA), 2023, pp.
  3635--3641.
\newblock \href {http://dx.doi.org/10.1109/ICRA48891.2023.10161372}
  {\path{doi:10.1109/ICRA48891.2023.10161372}}.

\bibitem{gerkey2004IJRR}
B.~P. Gerkey, M.~J. Matari{\'c}, A formal analysis and taxonomy of task
  allocation in multi-robot systems, The International Journal of Robotics
  Research 23~(9) (2004) 939--954.

\bibitem{korsah2013IJRR}
G.~A. Korsah, A.~Stentz, M.~B. Dias, A comprehensive taxonomy for multi-robot
  task allocation, The International Journal of Robotics Research 32~(12)
  (2013) 1495--1512.

\bibitem{zlot2006phd}
R.~Zlot, {An auction-based approach to complex task allocation for multirobot
  teams, PhD thesis}, Robotics Institute, Carnegie Mellon University, 5000
  Forbes Ave.

\bibitem{coppeliaSim}
E.~Rohmer, S.~P.~N. Singh, M.~Freese, {CoppeliaSim (formerly V-REP): a
  Versatile and Scalable Robot Simulation Framework}, in: Proc. of The
  International Conference on Intelligent Robots and Systems (IROS), 2013,
  www.coppeliarobotics.com.

\bibitem{sucan2012RAM}
I.~A. {\c{S}}ucan, M.~Moll, L.~E. Kavraki, The {O}pen {M}otion {P}lanning
  {L}ibrary, {IEEE} Robotics \& Automation Magazine 19~(4) (2012) 72--82,
  \url{https://ompl.kavrakilab.org}.
\newblock \href {http://dx.doi.org/10.1109/MRA.2012.2205651}
  {\path{doi:10.1109/MRA.2012.2205651}}.

\bibitem{darvish2020hierarchical}
K.~Darvish, E.~Simetti, F.~Mastrogiovanni, G.~Casalino, A hierarchical
  architecture for human--robot cooperation processes, IEEE Transactions on
  Robotics , to appear.

\bibitem{capitanelli2018RAS}
A.~Capitanelli, M.~Maratea, F.~Mastrogiovanni, M.~Vallati, On the manipulation
  of articulated objects in human--robot cooperation scenarios, Robotics and
  Autonomous Systems 109 (2018) 139--155.

\bibitem{darvish2018Mechatronics}
K.~Darvish, F.~Wanderlingh, B.~Bruno, E.~Simetti, F.~Mastrogiovanni,
  G.~Casalino, Flexible human--robot cooperation models for assisted shop-floor
  tasks, Mechatronics 51 (2018) 97--114.

\bibitem{karami2020task}
H.~Karami, K.~Darvish, F.~Mastrogiovanni, A task allocation approach for
  human-robot collaboration in product defects inspection scenarios, in: 2020
  29th IEEE International Conference on Robot and Human Interactive
  Communication (RO-MAN), IEEE, pp. 1127--1134.

\bibitem{Fristonetal2017}
K.~Friston, T.~Fitzgerald, F.~Rigoli, P.~Schwartenbeck, G.~Pezzulo, Active
  inference: a process theory, Neural Computation 29~(1).

\bibitem{BusemeyerandTownsend1993}
J.~Busemeyer, J.~Townsend, Decision field theory: a dynamic-cognitive approach
  to decision making in an uncertain environment, Psychological Review 100~(3).

\bibitem{Shaabanetal2023}
M.~Shaaban, A.~Carfì, F.~Mastrogiovanni, Digital twins for human-robot
  collaboration: a future perspective, in: Proceedings of the 18th
  International Conference on Intelligent Autononous Systems (IAS), Seoul,
  Corea, 2023.

\bibitem{Carfietal2018}
A.~Carfì, C.~Motolese, B.~Bruno, F.~Mastrogiovanni, Online human gesture
  recognition using recurrent neural networks and wearable sensors, in:
  Proceedings of the 27th IEEE International Conference on Robot and Human
  Interactive Communication (RO-MAN), Nanjing, China, 2018.

\bibitem{CarfiandMastrogiovanni2023}
A.~Carfì, F.~Mastrogiovanni, Gesture-based human-machine interaction:
  taxonomy, problem definition, and analysis, IEEE Transactions on Cybernetics
  53~(1).

\end{thebibliography}


\end{document}